\documentclass[twoside,11pt]{article}

\usepackage[abbrvbib, preprint]{jmlr2e}

\usepackage{amsmath}
\usepackage{amssymb}
\usepackage{amsfonts}
\usepackage{booktabs}
\usepackage{algorithm}
\usepackage{algpseudocode}
\usepackage{lastpage}

\let\proof\relax 
\usepackage{amsthm}
\renewenvironment{proof}{\par\noindent{\bf Proof\ }}{\hfill\BlackBox\\[2mm]}

\makeatletter
\let\c@theorem\relax
\let\c@lemma\relax
\let\c@proposition\relax
\let\c@remark\relax
\let\c@corollary\relax
\let\c@definition\relax
\let\c@conjecture\relax
\let\c@axiom\relax
\makeatother
\usepackage[capitalize,noabbrev]{cleveref}
\newtheoremstyle{jmlrplain}{}{}{\itshape}{}{\bfseries}{}{.5em}{}
\theoremstyle{jmlrplain}
\newtheorem{theorem}{Theorem}
\newtheorem{lemma}[theorem]{Lemma}
\newtheorem{proposition}[theorem]{Proposition}
\newtheorem{remark}[theorem]{Remark}
\newtheorem{corollary}[theorem]{Corollary}
\newtheorem{definition}[theorem]{Definition}
\newtheorem{conjecture}[theorem]{Conjecture}
\newtheorem{axiom}[theorem]{Axiom}
\newtheorem{qmcexample}[theorem]{Example}
\crefname{lemma}{Lemma}{Lemmas}
\crefname{theorem}{Theorem}{Theorems}
\crefname{proposition}{Proposition}{Propositions}
\crefname{corollary}{Corollary}{Corollaries}
\crefname{definition}{Definition}{Definitions}
\crefname{remark}{Remark}{Remarks}
\crefname{conjecture}{Conjecture}{Conjectures}
\crefname{axiom}{Axiom}{Axioms}
\crefname{qmcexample}{Example}{Examples}

\newcommand{\Z}{\mathbb{Z}}
\newcommand{\R}{\mathbb{R}}

\newcommand{\Fnx}{\mathbb{F}_n^\times}

\newcommand{\Res}{\operatorname{Res}}
\newcommand{\corr}{\operatorname{corr}}

\DeclareMathOperator{\Gal}{Gal}
\newcommand{\OK}{\mathcal O_K}

\jmlrheading{}{}{}{9/26}{xx}{xx}{Yueming LYU}

\ShortHeadings{Subgroup Rank-1 Lattices}{Yueming LYU}
\firstpageno{1}

\begin{document}

\title{ Subgroup Rank-1 Lattice  for
Practical High-dimensional Black-box Integral Approximation}

\author{\name Yueming LYU \email yueminglyu@gmail.com \\
       \addr Centre for Frontier AI Research (CFAR)
 \\
       \relax Agency for Science, Technology and Research (A*STAR) \\
       \relax 1 Fusionopolis Way \#16-16 Connexis  Singapore, \relax 138632}

\editor{TBD}

\maketitle

\begin{abstract}%
Estimating integrals of black-box, high-dimensional functions (expectations,
normalizing constants, kernel mean embeddings, or the softmax kernel inside
the self-attention of large language models) is a basic subroutine in
machine learning. Rank-1 lattice rules suit this setting because they query
the integrand only at a point set fixed in advance and need no gradients.
When the $n$ lattice points are also used as a design matrix
$X\in\R^{n\times d}$ for a feature map, however, applying an elementwise
nonlinearity $\Psi$ and then aggregating, $Y=\Psi(X)^\top v$, or expanding,
$u=\Psi(X)w$, costs $O(nd)$ time and $O(nd)$ memory for any standard
quasi-Monte Carlo point set. This becomes prohibitive when $n$ and $d$ are
both large. We study subgroup rank-1 lattices, whose Korobov power-form
generator $(1,t,\dots,t^{d-1})$ uses a scalar $t$ of fixed,
$n$-independent multiplicative order $m$. Splitting $\Fnx$ into the cosets
of $\langle t\rangle$ reduces both maps to short cyclic correlations
evaluated by FFT, so $Y$ and $u$ are computed exactly, for an arbitrary
$\Psi$, in $O(n\log m)$ time ($O(n\log d)$ when $m=\Theta(d)$) and $O(n)$
memory, without ever forming $X$. Fixing $m$ places the rule outside the
classical component-by-component averaging theory, so we prove convergence
directly: using resultants with the cyclotomic polynomial $\Phi_m$, we show
that for prime $m\ge d+1$ the squared worst-case error in the Korobov space
decays as $O(n^{-(\alpha-1)/(m-1)})$, and that this threshold is exact,
since for $m\le d$ the error has an $n$-independent floor. Using the
complete splitting of $n$ in $\mathbb Q(\zeta_m)$, we further show that
averaging over the $m-1$ admissible generators improves the constant by a
factor $\Theta(m-1)$. Empirically, the subgroup lattice is more accurate
than Gaussian random features, orthogonal random features, and scrambled
Sobol' and Halton points in $49$ of $54$ synthetic kernel-estimation
settings and in all $45$ softmax-attention settings on nine real embedding
datasets. At $d=2048$ and $n\approx4.1\times10^7$ it builds its sample set
in $2.3$\,ms, against $0.55$\,s for Gaussian sampling and minutes to hours
for the other baselines.
\end{abstract}


\section{Introduction}
\label{sec:intro}

Many computations in machine learning reduce to estimating an integral of
a black-box function over a high-dimensional domain: an expectation under
a model, a normalizing constant, a kernel mean embedding, or a
shift-invariant kernel written as an expectation over random frequencies
\citep{rahimi2007random}. The last case now sits inside large language
models, where linear-time attention mechanisms replace the softmax kernel
$\exp(\mathbf x^\top\mathbf y)$ by a Monte Carlo average of random
features \citep{choromanski2021rethinking}. In this setting the integrand
can be queried only through its values: gradients are unavailable or
expensive, and no tractable density is exposed. Sampling-based methods
such as Langevin and stochastic-gradient Monte Carlo
\citep{roberts1996exponential,welling2011bayesian} rely on exactly this
missing gradient information, and gradient-free Metropolis--Hastings
samplers \citep{robert2004monte} return a correlated, sequentially
generated chain rather than a fixed design. Quasi-Monte Carlo (QMC)
quadrature needs neither. It evaluates the integrand on a deterministic
point set fixed in advance, so the evaluations are embarrassingly parallel
and the same points can be reused across many integrands. Among QMC
constructions, the rank-1 lattice rule is one of the simplest and most
widely used: its $n$ points $\{l\mathbf z/n\}$, $l=0,\dots,n-1$, are
determined by a single generating vector $\mathbf z\in(\Z/n\Z)^d$, and a
well-chosen $\mathbf z$ attains a squared worst-case error of order
$n^{-\alpha+\delta}$, for any $\delta>0$, in the Korobov space of
smoothness $\alpha$ \citep{sloan2002component,kuo2003component,
dick2013high}.

In machine learning, however, a QMC point set is rarely used only to
average function values. Its $n$ points form a design matrix
$X\in\R^{n\times d}$ that is passed through a scalar nonlinearity $\Psi$
applied entrywise and then combined linearly with data. A QMC feature map
for a shift-invariant kernel \citep{avron2016quasi} is the standard
example, and the positive random features used for softmax attention
\citep{choromanski2021rethinking} have the same form. Two operations
recur: aggregating $n$ per-point values into $d$ per-coordinate summaries,
$Y=\Psi(X)^\top v$, and expanding $d$ coefficients to all $n$ points,
$u=\Psi(X)w$. Computed from their definitions, both cost $O(nd)$ time,
$O(nd)$ evaluations of $\Psi$, and $O(nd)$ memory to hold $X$ or
$\Psi(X)$. This is the cost for every standard point set, whether a
Halton or Sobol' sequence \citep{halton1960efficiency,sobol1967distribution}
or a rank-1 lattice with a generating vector found by the
component-by-component (CBC) algorithm
\citep{sloan2002component,kuo2003component,nuyens2006fast}, because none
of these constructions relates its $d$ coordinates to one another in a way
the computation could exploit. The product $nd$ is the problem. Accurate
estimates in high dimension need both $n$ and $d$ large at once, and at
$d=2048$ with $n\approx4.1\times10^7$, the largest setting in our
experiments, $X$ has about $8.4\times10^{10}$ entries, roughly
$670$\,GB in double precision. Direct evaluation is then not slow but
impossible, and even generating the point set takes minutes to hours for
scrambled Sobol', Halton, or orthogonal random features
(\Cref{sec:exp-construction-time}).

This paper shows that a specific structured rank-1 lattice removes the
factor of $d$ from both time and memory. We use the classical Korobov
power-form generator $\mathbf z(t)=(1,t,\dots,t^{d-1})\bmod n$
\citep{korobov1959approximate}, with $n$ prime, and require the scalar
$t\in\Fnx$ to have a small multiplicative order $m\mid n-1$ that is fixed
independently of $n$. We call the result a \emph{subgroup rank-1
lattice}. The restriction organizes $\Fnx$ into $q=(n-1)/m$ cosets of the
order-$m$ subgroup $H=\langle t\rangle$. Multiplication by $t$ acts on
each coset as a cyclic shift, so every coordinate of every lattice point
is an entry of one of $q$ short sequences of length $m$, and the
elementwise transform on each coset becomes a single length-$m$ cyclic
correlation, which an FFT evaluates in $O(m\log m)$ time. Processing the
cosets one at a time gives $Y$ and $u$ exactly, for an arbitrary $\Psi$,
in $O(n\log m)$ time and $O(n)$ memory, which is $O(n\log d)$ time when
$m=\Theta(d)$. The matrix $X$ is never formed.

The restriction has a price that the classical theory does not cover. As
$n$ grows, only the $\varphi(m)$ elements of order $m$ are admissible
values of $t$, a pool that stays fixed while $\Fnx$ grows. CBC would not
select such a generator, and the averaging arguments behind CBC
error bounds, which average over a candidate pool that grows with $n$,
cannot be applied. Whether a subgroup rank-1 lattice converges at all is
therefore a genuine question. We answer it directly. Writing the aliasing
condition $\mathbf h\cdot\mathbf z(t)\equiv0\pmod n$ as the vanishing of
an integer polynomial at a root of the cyclotomic polynomial $\Phi_m$
modulo $n$, and bounding the corresponding resultant, we show that for
prime $m\ge d+1$ the squared worst-case error decays as
$O(n^{-(\alpha-1)/(m-1)})$ for every admissible $t$, and that for
$m\le d$ it is bounded below by a constant for all $n$. This exponent is
weaker than the one CBC certifies for a generic generating vector. What
the theorem establishes is that the structure needed for the fast
transform is compatible with convergence, and it identifies the exact
dimension threshold at which convergence is lost. Using the complete
splitting of $n$ in the cyclotomic field $\mathbb Q(\zeta_m)$, we then
show that averaging over the $m-1$ admissible generators improves the
constant by a factor $\Theta(m-1)$, which justifies choosing the best
member of this small pool.

Empirically, the subgroup rank-1 lattice is both faster and more accurate
than the feature-map constructions in common use. With the best of the
$m-1$ generators selected once on a held-out batch and cached, it gives
the lowest error of five constructions in $49$ of $54$ synthetic
kernel-estimation settings and in all $45$ softmax-attention settings on
nine real embedding datasets, and it builds its complete sample set at
$d=2048$ and $n\approx4.1\times10^7$ in $2.3$\,ms.

\subsection{Contributions}
\label{sec:contributions}

\begin{enumerate}
\item \textbf{A fast, memory-light elementwise lattice transform
(\Cref{sec:fast}).} For a power-form generator $t$ of order $m$ and an
arbitrary map $\Psi:\R\to\R$, accessed only through its values, we show
that $Y=\Psi(X)^\top v$ and $u=\Psi(X)w$ can be computed exactly in
$O(n\log m)$ arithmetic operations, $n$ evaluations of $\Psi$, and $O(n)$
memory (\Cref{prop:fast-fft,prop:fast-Xw}; \Cref{alg:fast-transform,alg:fast-Xw}). Direct evaluation on any standard point set costs $O(nd)$
time, $O(nd)$ evaluations of $\Psi$, and $O(nd)$ memory. The argument is a
reindexing of the lattice by the cosets of $H=\langle t\rangle$ and uses
no property of $\Psi$, such as continuity or $\Psi(0)=0$. The same coset
structure means the whole sample set $\Psi(X)$ is determined by $n-1$
numbers rather than $nd$. Numerically, the fast transform agrees with
direct evaluation to floating-point precision and is $13$ to $35$ times
faster at $n\approx1.5\times10^6$ (\Cref{sec:fast-experiments}).

\item \textbf{Convergence of the subgroup rank-1 lattice, with an exact
threshold (\Cref{sec:rate}).} For prime $m\ge d+1$ and every $t$ of order
$m$, the squared worst-case error of the full $n$-point rule in the
Korobov space satisfies $e^2(Q_n,\mathbf z(t))\le C(d,\alpha,m)\,
n^{-(\alpha-1)/(m-1)}$ for all sufficiently large primes
$n\equiv1\pmod m$ (\Cref{thm:rate}). The proof is deterministic: it bounds
the resultant of the aliasing polynomial with $\Phi_m$, and needs no
averaging over a growing pool of generators. The condition $m\ge d+1$
cannot be relaxed: for $m\le d$ the coefficient vector of $\Phi_m$ is
itself an aliasing frequency, so $e^2(Q_n,\mathbf z(t))\ge2$ for every $n$
(\Cref{prop:floor,cor:threshold}).

\item \textbf{A sharper constant from the small candidate pool
(\Cref{sec:averaging}).} A frequency that aliases for $N_h$ of the $m-1$
admissible generators forces $n^{N_h}$, not merely $n$, to divide its
resultant (\Cref{lem:multiplicity}). The proof uses the complete splitting
of $n$ in $\mathbb Q(\zeta_m)$ and unique factorization of ideals in
$\mathbb Z[\zeta_m]$. Combined with a log-weighted tail bound
(\Cref{lem:log-tail}), this shows that the average squared worst-case
error over the $m-1$ generators is a factor $\Theta(m-1)$ below the
uniform bound of \Cref{thm:rate}, at the same exponent, so at least one
generator attains this sharper bound (\Cref{thm:average}).

\item \textbf{Experiments on kernel estimation and softmax attention
(\Cref{sec:experiments}).} We compare the subgroup rank-1 lattice with
Gaussian random features, orthogonal random features
\citep{yu2016orthogonal}, and scrambled Sobol' and Halton points, all
passed through the same nonlinearity. On estimation of
$\exp(\mathbf x^\top\mathbf y)$ with $d$ from $8$ to $2048$, it is the
most accurate method in $49$ of $54$ settings, with an advantage over the
best baseline that is largest at $d=2048$, reaching $13.5\times$ at the smallest and $5.5\times$ at the largest $n$ tested. On
self-normalized softmax attention over nine real embedding datasets, it
is the most accurate method in all $45$ settings (mean advantage
$1.67\times$). At $d=2048$ and $n\approx4.1\times10^7$ it builds its
sample set in $2.3$\,ms, against $0.55$\,s for Gaussian sampling and
$197$\,s to $2.1$\,h for the other baselines. The cost of these gains is
a one-time search over the $m-1$ candidate generators, paid once per
$(d,m,n)$ and cached.
\end{enumerate}

\paragraph{Scope and relation to earlier work.} This paper studies the
\emph{full} $n$-point lattice rule built from a general  fixed-order power-form
generator. Our earlier work \citep{lyu2020subgroup} used the  subgroup structure with specific $m=2d$ to 
to build a small closed-form point set by investigating the pairwise toroidal distance.  The fast transform,
the convergence theory, and the averaging result for the full lattice
developed here are new and do not follow from that construction.
Throughout, $n$ is prime; the convergence results of \Cref{sec:rate,sec:averaging} additionally assume that $m$ is prime.

\paragraph{Organization.} \Cref{sec:related} reviews related work and
\Cref{sec:prelim} recalls the Korobov space, worst-case error, and rank-1
lattice rules. \Cref{sec:subgroup-generator} defines the subgroup rank-1
lattice and its coset structure. \Cref{sec:fast} gives the fast
elementwise transform, \Cref{sec:rate} proves the convergence rate and the
exact threshold, and \Cref{sec:averaging} sharpens the constant by
averaging over the candidate pool. \Cref{sec:experiments} reports the
experiments and \Cref{sec:conclusion} concludes.

\section{Related Work}
\label{sec:related}

\paragraph{Rank-1 lattice rules and their construction.} Rank-1 lattice
rules go back to \citet{korobov1959approximate}; the classical theory is
covered by \citet{niederreiter1992random}, \citet{sloanjoe1994lattice},
and \citet{leobacher2014introduction}, and \citet{dick2013high} survey the
modern theory in the shift-invariant Korobov setting used here. The
standard way to choose a generating vector is the component-by-component
(CBC) algorithm \citep{sloan2002component,kuo2003component}, which fixes
one coordinate at a time by minimizing the worst-case error over all of
$\Fnx$. \citet{nuyens2006fast,nuyenscools2006fastrkhs} reduced its cost to
$O(dn\log n)$, and \citet{coolskuonuyens2006constructing} extended it to
embedded sequences of lattices; the \textsc{LatticeBuilder} software of
\citet{lecuyer2016latticebuilder} implements these searches, and
\citet{lecuyerlemieux2002recent} survey randomized variants. The error
bounds for CBC come from averaging the worst-case error over a pool of
candidates whose size grows with $n$. Which weighted function spaces make
this error independent of dimension is the subject of tractability theory
\citep{sloanwozniakowski1998when,sloan2001tractability,
hickernell2000integration,novak2008tractability}, with connections to
generalized discrepancy \citep{hickernell1998generalized} and ANOVA-type
decompositions \citep{kuo2010decompositions,kuo2011anziam}. Our
construction departs from this line in two ways. It restricts the
generating vector to a pool of fixed size $\varphi(m)$, so the averaging
argument is unavailable and convergence has to be proved by other means
(\Cref{sec:rate}). And it is chosen for the structure it gives the point
set rather than for the smallest worst-case error; its proven exponent is
weaker than the one CBC certifies.

\paragraph{Other low-discrepancy constructions.} The other main family of
QMC point sets consists of digital nets and sequences: Halton
\citep{halton1960efficiency}, Sobol'
\citep{sobol1967distribution,joekuo2003remark,joekuo2008constructing}, and
Faure \citep{faure1982discrepance} sequences, and the general
$(t,m,s)$-net framework of \citet{niederreiter1987point}. Their analysis
rests on classical equidistribution and discrepancy theory
\citep{weyl1916gleichverteilung,roth1954irregularities,
matousek1998l2discrepancy,matousek1999geometric}; see
\citet{niederreiterxing2001rational} for algebraic constructions and
\citet{dickpillichshammer2010digital} for a treatment that also covers
polynomial lattice rules \citep{dickkuopillichshammersloan2005construction}.
Randomized QMC, in particular Owen's scrambling
\citep{owen1995randomly,owen1997scrambled} and its higher-order
extensions \citep{dick2008walsh,dick2011higherorder}, adds an unbiased
error estimate at little cost in rate. None of these constructions relates
the coordinates of a point to one another, so pushing an elementwise
nonlinearity through $n$ points in $d$ dimensions costs $O(nd)$. Scrambled
Sobol' and Halton points are two of the baselines in
\Cref{sec:experiments}.

\paragraph{FFTs and rank-1 lattices.} Fast CBC \citep{nuyens2006fast}
already exploits multiplicative structure: reindexing $\Fnx$ by a
primitive root turns the CBC search over all $n-1$ candidates into a
circulant matrix--vector product evaluated by FFT. Our coset
decomposition is related in spirit but solves a different problem. Fast
CBC accelerates the \emph{search} for a generic generating vector; we fix
a structured generator and accelerate the \emph{use} of the resulting
point set. Rank-1 lattices are also used to sample and reconstruct
multivariate trigonometric polynomials, where a single one-dimensional
FFT of length $n$ recovers the Fourier coefficients on a frequency set
\citep{kammerer2014high,kammerer2015approximation,kammerer2015korobov,
kammerer2021highdim}, and fast CBC variants construct lattices for
function approximation rather than integration
\citep{kuo2009lattice,cools2021fast}. Those methods apply the FFT along
the point index $l$ to recover a function from its samples. Our transform
applies length-$m$ FFTs along the orbits of $\langle t\rangle$ in order to
apply an arbitrary nonlinearity $\Psi$ to the point set itself, a task
those methods do not address.

\paragraph{Random features and fast structured feature maps.} Random
Fourier features \citep{rahimi2007random} approximate a shift-invariant
kernel by a Monte Carlo average over random frequencies, and
\citet{avron2016quasi} replaced the random frequencies by QMC points to reduce the approximation error;
\citet{yang2014random} did the same for semigroup kernels. Random features
are now also used to linearize softmax attention
\citep{vaswani2017attention,choromanski2021rethinking}. Orthogonal random
features \citep{yu2016orthogonal} reduce variance by orthogonalizing
Gaussian directions, and are another baseline in \Cref{sec:experiments}.
Fast structured constructions reduce the $O(nd)$ cost of multiplying by a
dense random matrix: Fastfood \citep{le2013fastfood} and structured
orthogonal random features \citep{yu2016orthogonal} replace it by products
of Hadamard and random diagonal matrices, computed in $O(n\log d)$ time.
The subgroup rank-1 lattice reaches the same $O(n\log d)$ cost by a
different route. Its speed-up comes from the algebra of a deterministic
lattice rather than from a randomized matrix factorization, it applies to
an arbitrary nonlinearity $\Psi$ acting on the points, and the underlying
lattice rule comes with the worst-case error guarantees of
\Cref{sec:rate,sec:averaging}.

\paragraph{Algebraic number theory in lattice error analysis.} Our
convergence proof relates the aliasing condition to resultants of an
integer polynomial with the cyclotomic polynomial $\Phi_m$
\citep{apostol1970resultants,irelandrosen1990classical}, and the averaging
result uses the complete splitting of a prime $n\equiv1\pmod m$ in
$\mathbb Q(\zeta_m)$ and unique factorization of ideals in its ring of
integers \citep{washington1997introduction}. To our knowledge, neither
tool has previously been used to prove a convergence rate for a QMC rule.
Our earlier work \citep{lyu2020subgroup} used the  subgroup structure with specific $m=2d$ to 
to build a small closed-form point set by investigating the pairwise toroidal distance; the results of the
present paper concern the full $n$-point lattice with general $m$ and fast computation and convergence rate properties.

\paragraph{QMC and Markov chain Monte Carlo.} MCMC methods, including
Langevin and stochastic-gradient samplers
\citep{roberts1996exponential,welling2011bayesian} and gradient-free
Metropolis--Hastings \citep{robert2004monte}, are the default for
integrating against an unnormalized density, but they return a correlated
chain rather than a fixed, reusable design, and the scalable variants
need gradients. The two approaches have been combined:
low-discrepancy sequences can drive a Metropolis chain in place of i.i.d.\
uniforms \citep{owen2005quasi}, with discrepancy bounds for the resulting
estimators \citep{dick2016discrepancy}, and QMC has replaced random
sampling inside sequential Monte Carlo \citep{gerber2015sequential} and
approximate Bayesian computation \citep{buchholz2019improving}. Our
contribution is complementary: it makes a fixed lattice design cheaper to
use at large $n$ and $d$, and could be used inside such hybrid schemes.

\paragraph{High-dimensional applications of QMC.} QMC was first shown to
beat Monte Carlo on high-dimensional problems in computational finance
\citep{paskov1995faster}, an effect explained by the low effective
dimension of the integrands \citep{caflisch1997valuation}; see
\citet{lecuyer2004quasi} for a survey. Rank-1 lattice rules are now
standard in uncertainty quantification for PDEs with random coefficients,
where $d$ can reach the thousands
\citep{graham2011qmcpde,kuo2012qmcfem,dick2014nonperiodic,kuo2016survey},
including multilevel variants \citep{kuo2017multilevel} and eigenvalue
problems \citep{gilbert2024eigenvalue}; the MCQMC proceedings
\citep[e.g.][]{cools2016monte} track the field. These works control the
cost of large $d$ through weighted spaces, dimension truncation, and
multilevel variance reduction. The fast transform developed here attacks
a different cost, the $O(nd)$ cost of evaluating an elementwise map on the
point set, and can be combined with those techniques.

\section{Preliminaries}
\label{sec:prelim}

Throughout, $n$ is prime and $\mathbb F_n = \mathbb Z/n\mathbb Z$, so that
$\mathbb F_n^\times$ is cyclic of order $n-1$. We fix an integer $d \ge 1$
(the dimension) and a real smoothness parameter $\alpha > 1$, and write
$e(x) := e^{2\pi i x}$.

\subsection{The Korobov space and worst-case error}

\begin{definition}[Korobov space]
\label{def:korobov}
For $\mathbf k \in \mathbb Z^d$ let $\rho(\mathbf k) :=
\prod_{j=1}^d \max(1,|k_j|)^\alpha$, so that $\rho(\mathbf 0) = 1$. Define
\begin{equation}
    \mathcal{H}_d^\alpha := \Big\{ f \in L^2([0,1)^d) : \|f\|^2 :=
\sum_{\mathbf k \in \mathbb Z^d} |\hat f(\mathbf k)|^2 \rho(\mathbf k)
< \infty \Big\}, \qquad
\hat f(\mathbf k) := \int_{[0,1)^d} f(\mathbf x)\, e(-\mathbf k\cdot
\mathbf x)\, d\mathbf x.
\end{equation}
Since $\alpha > 1$, $\sum_{k\in\mathbb Z} \max(1,|k|)^{-\alpha} = 1 +
2\zeta(\alpha) < \infty$, so $\mathcal{H}_d^\alpha$ is a reproducing-kernel Hilbert
space with kernel
\begin{equation}
    K(\mathbf x, \mathbf y) = \sum_{\mathbf k \in \mathbb Z^d} \rho(\mathbf
k)^{-1}\, e\big(\mathbf k \cdot (\mathbf x - \mathbf y)\big).
\end{equation}
\end{definition}

For a node set $X = \{\mathbf x_1, \dots, \mathbf x_N\} \subset [0,1)^d$,
the associated equal-weight quadrature rule is $Q_X(f) := \frac1N
\sum_{l=1}^N f(\mathbf x_l)$, approximating $I(f) := \hat f(\mathbf 0) =
\int_{[0,1)^d} f(\mathbf x) d\mathbf x$. Its worst-case error over the unit ball of $\mathcal{H}_d^\alpha$
is
\begin{equation}
   e(Q_X, \mathcal{H}_d^\alpha) := \sup_{\|f\| \le 1} |Q_X(f) - I(f)|. 
\end{equation}

\begin{lemma}[Worst-case-error representation]
\label{lem:wce}
For any node set $X = \{\mathbf x_1, \dots, \mathbf x_N\}$,
\begin{equation}
  e^2(Q_X, \mathcal{H}_d^\alpha) = \sum_{\mathbf k \ne \mathbf 0} \rho(\mathbf k)^{-1}
\Big| \frac1N \sum_{l=1}^N e(\mathbf k \cdot \mathbf x_l) \Big|^2 .  
\end{equation}
\end{lemma}

\begin{proof}
By the reproducing property, $f(\mathbf y) = \langle f, K(\cdot,\mathbf
y)\rangle$, so $Q_X(f) = \langle f, \frac1N \sum_l K(\cdot, \mathbf
x_l)\rangle$. Since $\int_{[0,1)^d} e(\mathbf k\cdot \mathbf x)\, d\mathbf
x = \mathbf 1[\mathbf k = \mathbf 0]$, integrating $K(\mathbf x,\mathbf y)$
over $\mathbf x$ kills every Fourier term except $\mathbf k=\mathbf 0$, so
$\int_{[0,1)^d} K(\mathbf x,\mathbf y)\, d\mathbf x \equiv 1$. 

Let $1(\cdot)$ denotes the constant function equal to 1 everywhere. Then, its Fourier coefficients are $\hat{1}(\boldsymbol{k})= \int_{[0,1)^d} e(-\mathbf k\cdot \mathbf x)\, d\mathbf
x = \mathbf 1[\mathbf k = \mathbf 0]$.
Then, we have that
\begin{align}
    \langle f,1(\cdot) \rangle= \sum_{\mathbf k \in \mathbb Z^d} \hat f(\mathbf k) \hat{1}(\boldsymbol{k}) \rho(\mathbf k) = \hat f(\mathbf 0)\rho(\mathbf 0) = \hat f(\mathbf 0) = I(f) \nonumber
\end{align}

So the constant function $1(\cdot)$ represents $I$ under the inner product of $\mathcal{H}_d^\alpha$. Hence $Q_X(f) - I(f) = \langle f, \frac1N \sum_l K(\cdot, \mathbf
x_l) - 1(\cdot) \rangle =   \langle f, \gamma(\cdot)\rangle$. Then
\[
\gamma(\mathbf y) = \frac1N \sum_l K(\mathbf x_l,\mathbf y) - 1 =
\sum_{\mathbf k \ne \mathbf 0} \rho(\mathbf k)^{-1} e(-\mathbf k \cdot
\mathbf y) \Big( \frac1N \sum_l e(\mathbf k\cdot \mathbf x_l) \Big),
\]
the $\mathbf k = \mathbf 0$ terms of the two pieces having cancelled
exactly. By Cauchy--Schwarz, $e(Q_X,\mathcal{H}_d^\alpha) = \sup_{\|f\|\le 1}
|\langle f, \gamma(\cdot)\rangle| = \|\gamma(\cdot)\|$, and Parseval's theorem applied to the last display
gives the claim.
\end{proof}

\subsection{Integration lattices and generator matrices}
\label{sec:lattice-prelim}

We first recall the classical notion of a lattice rule
\citep{niederreiter1992random,sloanjoe1994lattice,dick2013high}, of which
the rank-1 construction used throughout this paper is the simplest
nontrivial special case.

\begin{definition}[Lattice, generator matrix]
\label{def:lattice}
A \emph{lattice} $\Lambda \subset \R^d$ is a set of the form $\Lambda = B
\Z^d := \{B\mathbf k : \mathbf k \in \Z^d\}$ for some nonsingular matrix
$B \in \R^{d\times d}$, called a \emph{generator matrix} of $\Lambda$. Two
matrices generate the same lattice, $B\Z^d = B'\Z^d$, if and only if $B' =
BU$ for some $U$ in $GL_d(\Z)$ (an integer matrix with $\det U = \pm1$);
in particular $|\det B|$ depends only on $\Lambda$, not on the choice of
$B$. An \emph{integration lattice} is a lattice with $\Z^d \subseteq
\Lambda$; equivalently, some (any) generator matrix $B$ can be chosen
with $B^{-1}$ an integer matrix, and then $N := |\det B|^{-1} = [\Lambda :
\Z^d]$ (the index of $\Z^d$ in $\Lambda$, i.e.\ the order of the finite
quotient group $\Lambda/\Z^d$) is a positive integer.
\end{definition}

Because $\Z^d \subseteq \Lambda$, translating any point of $\Lambda$ by an
integer vector remains in $\Lambda$, so $\Lambda \cap [0,1)^d$,  the
lattice points that lie \emph{inside the unit cube},  is a set of exactly
$N$ representatives, one for each coset of $\Z^d$ in $\Lambda$ (a
fundamental domain for $\Lambda/\Z^d$).

\begin{definition}[Lattice rule]
\label{def:lattice-rule}
For an integration lattice $\Lambda$ of index $N$, the associated
(equal-weight) \emph{lattice rule} is $Q_\Lambda(f) := \frac1N
\sum_{\mathbf x \in \Lambda \cap [0,1)^d} f(\mathbf x)$.
\end{definition}

Three properties of this construction are used repeatedly below.

\begin{enumerate}
\item[(i)] \textbf{Periodicity and well-posedness.} $Q_\Lambda$ depends
only on $\Lambda$, not on the particular choice of the $N$ coset
representatives used to enumerate $\Lambda \cap [0,1)^d$: for a
$\Z^d$-periodic function (such as any $f \in \mathcal{H}_d^\alpha$, via its Fourier
series), summing over any fundamental domain of $\Lambda/\Z^d$ gives the
same value.

\item[(ii)] \textbf{Duality and the aliasing set.} The \emph{dual lattice}
$\Lambda^* := \{\mathbf k \in \R^d : \mathbf k \cdot \mathbf x \in \Z\
\forall\, \mathbf x \in \Lambda\}$ satisfies $\Lambda^* \subseteq \Z^d$
whenever $\Z^d \subseteq \Lambda$ (duality reverses inclusions, and
$(\Z^d)^* = \Z^d$). Equivalently, $\Lambda^*$ consists of exactly the
frequencies $\mathbf k \in \Z^d$ for which the character
$e(\mathbf k \cdot \mathbf x)$ is identically $1$ on $\Lambda$. This is
precisely the set that controls the Fourier/worst-case-error analysis of
\Cref{lem:wce}: applying that lemma with $N = |\Lambda \cap [0,1)^d|$ and
the geometric-sum identity $\frac1N \sum_{\mathbf x \in \Lambda\cap
[0,1)^d} e(\mathbf k\cdot \mathbf x) = \mathbf 1[\mathbf k \in \Lambda^*]$
shows that only $\Lambda^* \setminus \{\mathbf 0\}$ contributes to
$e^2(Q_\Lambda, \mathcal{H}_d^\alpha)$:
\begin{equation}
e^2(Q_\Lambda, \mathcal{H}_d^\alpha) = \sum_{\mathbf k \in \Lambda^* \setminus \{\mathbf
0\}} \rho(\mathbf k)^{-1}.
\label{eq:aliasing-general}
\end{equation}

\item[(iii)] \textbf{Rank.} The finite abelian group $\Lambda/\Z^d$ is,
by the structure theorem, a product of $r$ nontrivial cyclic groups
$\Z/n_1 \times \cdots \times \Z/n_r$ ($n_1,\dots,n_r \ge 2$, $N =
n_1\cdots n_r$) for a unique $r \ge 0$ (the elementary-divisor/Smith
normal form of $B^{-1}$); $r$ is the \emph{rank} of the lattice rule, and
each cyclic factor is generated by the class of some integer vector
$\mathbf z_i/n_i \bmod \Z^d$. Rank $0$ is $\Lambda = \Z^d$ itself ($N=1$,
the trivial rule); this paper is concerned exclusively with the simplest
nontrivial case, $r = 1$.
\end{enumerate}

\subsection{Rank-1 lattice rules and the power-form generator}
\label{sec:rank1}

A \emph{rank-1} integration lattice is generated by a single vector: for
$\mathbf z \in \Z^d$ and $n \ge 1$,
\begin{equation}
\Lambda(\mathbf z, n) := \tfrac1n \Z\mathbf z + \Z^d = \Big\{ \tfrac{l}{n}
\mathbf z + \mathbf k : l \in \Z,\ \mathbf k \in \Z^d \Big\}.
\label{eq:rank1-lattice}
\end{equation}
Only $\mathbf z \bmod n$ matters, so we may (and do) take $\mathbf z \in
(\Z/n\Z)^d$. The $N = n$ points of $\Lambda(\mathbf z,n) \cap [0,1)^d$
are $\mathbf x_l := \{l\mathbf z/n\}$ for $l = 0,\dots,n-1$
(coordinatewise fractional part), which recovers the familiar rank-1
lattice rule:

\begin{definition}[Full rank-1 lattice rule]
For a generating vector $\mathbf z \in (\mathbb Z/n\mathbb Z)^d$, the
\emph{full} rank-1 lattice rule $Q_n$ uses all $n$ points $\mathbf x_l :=
\{l\mathbf z/n\}$, $l = 0,\dots,n-1$, with $Q_n(f) := \frac1n
\sum_{l=0}^{n-1} f(\mathbf x_l)$.
\end{definition}

Here \eqref{eq:aliasing-general} specializes, via
$\Lambda(\mathbf z,n)^* = \{\mathbf t \in \Z^d : \mathbf t \cdot \mathbf z
\equiv 0 \!\!\pmod n\}$, to the aliasing identity
\begin{equation}
e^2(Q_n, \mathcal{H}_d^\alpha) = \sum_{\substack{\mathbf t \ne \mathbf 0 \\
\mathbf t \cdot \mathbf z \equiv 0 \,(n)}} \rho(\mathbf t)^{-1},
\label{eq:aliasing}
\end{equation}
which is the basis for the convergence analysis of \Cref{sec:rate}.

\paragraph{Generator matrix for a rank-1 lattice.} A rank-1 lattice rule
is genuinely rank $1$, rather than a rank-$0$ rule in disguise,  exactly
when $\gcd(z_1,n) = 1$ for at least one coordinate $z_i$; relabelling
coordinates and replacing $\mathbf z$ by $z_i^{-1}\mathbf z \bmod n$ (a
mere re-indexing $l \mapsto l\, z_i \bmod n$ of the same point set) lets us
take $z_1 = 1$ without loss of generality, exactly the normalization
used by the power-form generator below. With $z_1 = 1$, an explicit
generator matrix realizing \Cref{def:lattice} for
$\Lambda(\mathbf z, n)$ is the lower unitriangular matrix
\begin{equation}
B(\mathbf z, n) = \frac1n
\begin{pmatrix}
1 & 0 & \cdots & 0 \\
z_2 & n & \cdots & 0 \\
\vdots & & \ddots & \\
z_d & 0 & \cdots & n
\end{pmatrix}, \qquad \det B(\mathbf z,n) = \frac1n,
\label{eq:rank1-generator-matrix}
\end{equation}
so that $B(\mathbf z,n)\Z^d = \frac1n \Z\mathbf z + \Z^d =
\Lambda(\mathbf z,n)$, confirming index $N=n$ directly from
\Cref{def:lattice} (this is the Hermite normal form of $B^{-1}$, unique up
to the right-$GL_d(\Z)$ freedom already noted). Every other valid
generator matrix for $\Lambda(\mathbf z,n)$ is $B(\mathbf z,n)U$ for some
$U \in GL_d(\Z)$.

Throughout, we work with the classical Korobov \emph{power-form}
generator, in which the $d$ coordinates of $\mathbf z$ are consecutive
powers of a single scalar,  a specific, highly structured construction
rule for the generating vector $\mathbf z$ (equivalently, via
\eqref{eq:rank1-generator-matrix}, for the generator matrix $B$ itself),
rather than an arbitrary rank-1 vector as in \eqref{eq:rank1-lattice}.

\begin{definition}[Power-form generator]
For $t \in \mathbb F_n^\times$ and $d \ge 1$, write $\mathbf z(t) :=
(1, t, t^2, \dots, t^{d-1}) \bmod n \in (\mathbb Z/n\mathbb Z)^d$.
\end{definition}

\section{Subgroup rank-1 lattice}
\label{sec:subgroup-generator}

The component-by-component algorithm chooses a generating vector,  power
form or otherwise,  to directly optimize $e(Q_n, \mathcal{H}_d^\alpha)$, and the
resulting optimal $t$ (or optimal $\mathbf z$) is generic: it depends on
$n$ and $d$ in a way that carries no simpler algebraic description. This
section instead studies the family obtained by fixing the
\emph{multiplicative order} of $t$, independently of $n$,  namely,  
\emph{subgroup rank-1 lattice} construction that gives this paper its
name,  and works out, in some detail, the group-theoretic structure that
this restriction imposes on the generating vector and on $\mathbb
F_n^\times$ itself. That structure is exactly what \Cref{sec:fast} and
\Cref{sec:rate} exploit, respectively, for fast computation and for convergence.

\subsection{Fixed-order power-form generators}
\label{sec:subgroup-def}

\begin{definition}[Fixed-order power-form generator]
\label{def:fixed-order}
Fix a divisor $m$ of $n-1$ with $m \ge d+1$. Since $\mathbb F_n^\times$ is
cyclic, it has a unique subgroup $H \le \mathbb F_n^\times$ of order $m$,
and $H = \langle t\rangle$ for any $t \in \mathbb F_n^\times$ of
multiplicative order exactly $m$. We study the power-form generator
$\mathbf z(t)$ for such a $t$, with $m$ held fixed as $n \to \infty$ (so
that, as $n$ grows over primes with $m \mid n-1$, only the $\varphi(m)$
generators of $H$ are ever used). We impose $m \ge d+1$ throughout the
construction,  rather than the weaker $d \le m$ that suffices merely for
the power-form generator to be well posed (\Cref{sec:fast-setting}), 
because $m \ge d+1$ is exactly the condition under which \Cref{thm:rate}
proves the resulting lattice rule converges, and \Cref{cor:threshold}
shows this threshold is exact: convergence provably fails once $m \le d$.
\end{definition}

Writing $g$ for a fixed primitive root modulo the prime $n$, the elements
of $H$ of order exactly $m$ are the $g^{k(n-1)/m}$ with $\gcd(k,m) = 1$;
taking $t = g^{(n-1)/m}$ (the case $k=1$, without loss of generality up to
relabelling the generator) gives the closed-form generating vector
\begin{equation}
\mathbf z(t) = \Big( g^{0},\, g^{(n-1)/m},\, g^{2(n-1)/m}, \dots,
g^{(d-1)(n-1)/m} \Big) \bmod n, \qquad m \mid (n-1),\ \ m \ge d+1,
\label{DualGroup}
\end{equation}
 which yields a full rank-1 lattice rule via
\eqref{eq:rank1-lattice}.

Two elementary facts about cyclic groups make \Cref{def:fixed-order} well
posed and worth isolating explicitly, since both are used repeatedly below
without further comment. First, \emph{existence and uniqueness of $H$}:
because $\mathbb F_n^\times$ is cyclic of order $n-1$, its subgroups are in
order-preserving bijection with the divisors of $n-1$, one subgroup per
divisor, so ``the'' subgroup of order $m$ is unambiguous exactly when $m
\mid n-1$,  the divisibility hypothesis in \Cref{def:fixed-order} is
therefore not a simplifying assumption but a necessary condition for $H$
to exist at all. Second, \emph{the true size of the generator pool}: an
element $t \in H$ generates $H$ (rather than a proper subgroup of it) iff
$t$ has order exactly $m$, and a cyclic group of order $m$ has exactly
$\varphi(m)$ such elements ($\varphi$ the Euler totient), so $\varphi(m)$, not $m$, and certainly not $n$,  is the true size of the candidate pool
underlying \Cref{def:fixed-order}. This is a severe restriction: as $n \to
\infty$ along primes with $m \mid n-1$, the pool of admissible $t$ does not
grow at all, in sharp contrast to a CBC-style search, which optimizes over
a pool of candidate generators that grows with $n$.

\begin{qmcexample}[A subgroup of prime order $5$]
\label{ex:subgroup}
Take $n = 11$, so $\mathbb F_{11}^\times$ is cyclic of order $10$ and
$m = 5$ is a valid choice since $5 \mid 10$; note $m$ is itself prime.
Fixing the primitive root $g = 2$ modulo $11$ (indeed $2^{10} \equiv 1$
and no smaller power does), direct computation gives $t = g^{(n-1)/m} =
2^2 = 4$, which has order $5$: $4^1 = 4$, $4^2 = 16 \equiv 5$, $4^3
\equiv 9$, $4^4 \equiv 3$, $4^5 \equiv 1 \pmod{11}$. Hence $H = \langle 4
\rangle = \{1,4,5,9,3\}$ is the unique subgroup of order $5$. Because $m$
is prime, $H$ has no nontrivial proper subgroups (Lagrange's theorem
leaves only the divisors $1$ and $m$ of $m$ itself), so \emph{every} nontrivial element of $H$ already
has order exactly $m$ and so generates all of $H$: the full generator
pool is $\varphi(5) = 4$ elements, namely all of $H \setminus \{1\} =
\{3,4,5,9\}$, with none wasted on smaller subgroups. Taking $d = 3$, so
that $m = 5 \ge d+1 = 4$,  comfortably inside the construction's
requirement (\Cref{def:fixed-order}). Two of the four admissible fixed-order power-form
generators for this $(n,m,d)$ are
\[
\mathbf z(4) = (1,4,5), \qquad \mathbf z(3) = (1,3,9) \pmod{11},
\]
each giving a full $11$-point rank-1 lattice rule via
\eqref{eq:rank1-lattice};  the two vectors already differ in
their second coordinate, so, despite generating the same subgroup $H$,  $t=4$ and $t=3$ do not produce the same lattice, only two (of four)
lattices built from the same underlying group of prime order.
\end{qmcexample}

\subsection{The coset decomposition of $\mathbb F_n^\times$ induced by
$H$}
\label{sec:subgroup-cosets}

The single group-theoretic fact underlying both the fast algorithm of
\Cref{sec:fast} and the convergence analysis of \Cref{sec:rate} is that
multiplication by $t$ organizes all of $\mathbb F_n^\times$, not just $H$
itself, into $H$-orbits of a very simple, explicit shape.

\begin{lemma}[Coset decomposition and cyclic shift structure]
\label{lem:coset-decomp}
Let $t \in \mathbb F_n^\times$ have order exactly $m > 1$, let $H =
\langle t \rangle$, and let $q := (n-1)/m$. Then:
\begin{enumerate}
\item[(i)] $\mathbb F_n^\times$ partitions into exactly $q$ cosets of $H$,
each of size $m$;
\item[(ii)] the map $\pi : r \mapsto rt \bmod n$ permutes $\mathbb
F_n^\times$, has no fixed points, and its orbits are exactly these $q$
cosets; restricted to a single coset $C = rH = \{r, rt, \dots, rt^{m-1}\}$,
$\pi$ acts as the cyclic shift $rt^i \mapsto rt^{i+1 \bmod m}$.
\end{enumerate}
\end{lemma}

\begin{proof}
(i) is Lagrange's theorem applied to $H \le \mathbb F_n^\times$: the
cosets of $H$ partition $\mathbb F_n^\times$ into classes of size
$|H| = m$, and there are $|\mathbb F_n^\times|/|H| = (n-1)/m = q$ of them.
For (ii), $\pi$ is a bijection of $\mathbb F_n^\times$ because $t$ is
invertible mod $n$. For any $r \in \mathbb F_n^\times$ and $k \ge 1$,
$\pi^k(r) = r t^k$, so $\pi^k(r) = r$ iff $t^k \equiv 1 \pmod n$ iff $m
\mid k$, since $t$ has order exactly $m$; taking $k=1$ shows $\pi$ has no
fixed points, and taking $k=m$ (the least such $k$) shows the orbit of $r$
under $\pi$ has size exactly $m$. That orbit is $\{r, rt, \dots,
rt^{m-1}\} = rH$, a coset of size $m$ equal to the orbit itself, so $\pi$
maps $rH$ to itself as a single $m$-cycle, i.e.\ the stated cyclic shift,
rather than splitting $rH$ into shorter orbits.
\end{proof}

\begin{remark}
\Cref{lem:coset-decomp} is exactly the mechanism used, with $\pi$ taken to
be multiplication by $t^{-1}$ rather than $t$, in the proof of
\Cref{prop:fast-fft}: the coset-FFT decomposition of \Cref{sec:fast}
applies a length-$m$ FFT once per coset because \Cref{lem:coset-decomp}(i)
already guarantees there are only $q = (n-1)/m$ cosets to process, and it
is valid to treat each coset's contribution as a single length-$m$
\emph{cyclic} correlation only because \Cref{lem:coset-decomp}(ii) says
the relevant reindexing is genuinely one $m$-cycle on that coset, not some
other permutation of the same $m$ points. The same partition of $\mathbb
F_n^\times$ into $q$ cosets of $H$ reappears in \Cref{sec:rate} in a
different guise, indexing the aliasing frequencies that a fixed
subgroup-order generator can and cannot kill (\Cref{eq:aliasing}).
\end{remark}

\subsection{Scope of the restriction}
\label{sec:subgroup-scope}

\Cref{def:fixed-order} is a severe, non-generic restriction on the
generating vector, and it is not something CBC would produce: as $n \to
\infty$ the pool of admissible $t$ does not grow at all,  it is pinned at
$\varphi(m)$ elements, all lying in the single fixed subgroup $H$, however
large $n$ becomes. The remainder of the paper works out what this buys and
what it costs. \Cref{sec:fast} shows this restriction buys a genuine
computational advantage for the elementwise lattice transform underlying
feature maps built on the lattice, via exactly the coset structure of
\Cref{lem:coset-decomp}. \Cref{sec:rate} shows, despite the classical
average-case convergence theory not covering a pool of candidates that
stays fixed in size,  that the restriction is compatible with convergence
of the resulting quadrature rule, at an explicit (if not provably tight)
polynomial rate, again by way of the algebraic structure of $H$ (now
viewed through the $m$-th cyclotomic polynomial rather than through the
coset decomposition above). \Cref{sec:averaging} then asks how much
averaging over the small pool of $\varphi(m)$ candidates,  all we have, in
place of a pool that would ordinarily grow with $n$, can sharpen that
rate's constant.

\section{A fast $O(n\log d)$ algorithm for the elementwise lattice transform}
\label{sec:fast}

This section is the technical core of the paper: it gives a fast algorithm
for the elementwise lattice transform $\Psi(X)^\top v$ and $\Psi(X)w$ that
underlies lattice-based feature maps, for an \emph{arbitrary} scalar map
$\Psi$, with no assumption on $\Psi$ anywhere in the argument. This is a
computational advantage of the fixed-order construction that is
independent of, and complementary to, the convergence question of
\Cref{sec:rate}.

\subsection{Setting}
\label{sec:fast-setting}

Let $n$ be prime, let $t \in \mathbb F_n^\times$ have multiplicative order
exactly $m \mid n-1$, and let $q := (n-1)/m$. Let $X \in \mathbb
R^{n\times m}$ be the full $n$-point lattice-point matrix over all $m$
columns compatible with $t$,
\[
X_{l,j} := \{ l\, t^{j-1}/n \}, \qquad l = 0,\dots,n-1,\ \ j = 1,\dots,m,
\]
so that the first $d \le m -1$ columns of $X$ are exactly the coordinates of
the $n$ points underlying $Q_n$ in \Cref{sec:rank1} (recall $d \le m -1$ is
required there for the power-form construction to be well-posed, i.e.\ for
$t^0,\dots,t^{d-1}$ to be pairwise distinct within one period of $H =
\langle t \rangle$).

Fix an \emph{arbitrary} elementwise map $\Psi : \mathbb R \to \mathbb R$,
this is the operative generality of this section: no continuity,
linearity, monotonicity, or $\Psi(0) = 0$ assumption is made anywhere
below,  and write $\Psi(X) \in \mathbb R^{n\times m}$ for its entrywise
action, $\Psi(X)_{l,j} := \Psi\big(\{l\,t^{j-1}/n\}\big)$. For $v \in
\mathbb R^n$ (e.g.\ $n$ function evaluations, or quadrature weights to be
aggregated) and $w \in \mathbb R^m$ (e.g.\ $m$ per-coordinate
coefficients), our two objects of study are the aggregation and expansion
maps
\[
Y := \Psi(X)^\top v \in \mathbb R^m, \qquad u := \Psi(X)\, w \in \mathbb
R^n.
\]
Direct evaluation of either from its definition costs $O(nm)$ arithmetic
operations and $O(nm)$ evaluations of $\Psi$,  the same cost incurred when
$X$ instead collects the coordinates of a Halton sequence, a Sobol'
sequence, or a rank-1 lattice with a generic (non-power-form) generating
vector, since none of these point sets admit a faster elementwise
transform; this $O(nm)$ direct-evaluation cost is the baseline against
which the rest of this section is measured. As the proofs below make
clear, the coset-FFT saving is a pure reindexing argument that never
invokes linearity, continuity, or any other structural property of
$\Psi$: it is exactly as insensitive to $\Psi$ as it is to the numerical
entries of $v$ or $w$, so it applies uniformly across every choice of
$\Psi$ at once, with no distinguished case. This is precisely the
primitive needed to evaluate a quasi-Monte-Carlo feature map
\citep{avron2016quasi},  where a fixed scalar nonlinearity is applied to
every lattice coordinate before the linear aggregation or expansion step.

\subsection{Main result: a coset-FFT decomposition}

\begin{proposition}[Fast elementwise lattice transform]
\label{prop:fast-fft}
In the setting above, for any $\Psi : \mathbb R \to \mathbb R$ and any $v
\in \mathbb R^n$, $Y = \Psi(X)^\top v \in \mathbb R^m$ can be computed
exactly in $O(n\log m)$ arithmetic operations plus $n$ evaluations of
$\Psi$,  $O(n\log d)$ arithmetic operations once $m = \Theta(d)$,   as
against $O(nm)$ arithmetic operations plus $nm$ evaluations of $\Psi$ for
direct evaluation.
\end{proposition}

\begin{proof}
Fix $j \in \{1,\dots,m\}$ and expand
\[
Y_j = \sum_{l=0}^{n-1} v_l\, \Psi\big(\{l\, t^{j-1}/n\}\big).
\]
Since $t^{j-1} \in \mathbb F_n^\times$, the map $l \mapsto r := l\,
t^{j-1} \bmod n$ is a bijection of $\{0,\dots,n-1\}$ with inverse $l = r\,
t^{-(j-1)} \bmod n$; substituting $r$ for $l$,
\begin{equation}
Y_j = \sum_{r=0}^{n-1} \Psi(w_r)\, v_{r\, t^{-(j-1)} \bmod n}, \qquad
w_r := \frac{r}{n}, \quad \pi(r) := r\, t^{-1} \bmod n.
\label{eq:fast-substitution}
\end{equation}
Unlike when $\Psi = \mathrm{id}$, $\Psi(w_0) = \Psi(0)$ need not vanish,
so the fixed point $r=0$ of $\pi$ now contributes to
\eqref{eq:fast-substitution}; separate it out and apply the coset
decomposition to the remaining $r \in \mathbb F_n^\times$. Restricted to
$\mathbb F_n^\times$, $\pi$ is multiplication by the order-$m$ element
$t^{-1}$; its orbits there are exactly the $q$ cosets of $H = \langle
t\rangle$, each of size $m$. Fix one coset $C = \{c, ct^{-1}, \dots,
ct^{-(m-1)}\}$ and set $a_i := \Psi(w_{c t^{-i}})$, $b_i := v_{c t^{-i}}$
for $i = 0,\dots,m-1$ (indices cyclic mod $m$, since $t^{-m} = 1$).
Because $\pi^{j-1}(ct^{-i}) = ct^{-(i+j-1)}$,
\[
\sum_{r \in C} \Psi(w_r)\, v_{\pi^{j-1}(r)} = \sum_{i=0}^{m-1} a_i\,
b_{(i+j-1)\bmod m} = \corr(a,b)[j-1],
\]
the length-$m$ cyclic correlation of $a$ and $b$ evaluated at shift
$j-1$. Summing over the $q$ cosets and adding back the $r=0$ term,
\begin{equation}
Y_j = \Psi(0)\, v_0 + \sum_{s=1}^q \corr\big(a^{(s)}, b^{(s)}\big)[j-1],
\qquad j = 1,\dots,m.
\label{eq:coset-sum}
\end{equation}
Each of the $q$ cyclic correlations is a single length-$m$ FFT pair,
$\corr(a,b) = \mathrm{IFFT}\big(\overline{\mathrm{FFT}(a)} \cdot
\mathrm{FFT}(b)\big)$, costing $O(m\log m)$ \citep{cooley1965algorithm};
all $m$ of its output shifts are read out in \eqref{eq:coset-sum}, since
$Y$ has $m$ entries. Forming each $a^{(s)}$ costs $m$ evaluations of
$\Psi$ ($n-1$ in total across the $q$ cosets), which changes none of the
$O(m\log m)$ arithmetic cost of the FFT pair used to evaluate
$\corr(a^{(s)}, b^{(s)})$. The total arithmetic cost is $O(q\cdot
m\log m) + O(n) = O\big((n-1)\log m\big) + O(n) = O(n\log m)$, the
additional $O(n)$ term covering construction of the $q$ cosets and the
final sum \eqref{eq:coset-sum}; taking $m = \Theta(d)$ gives
$O(n\log d)$. On top of this, forming the $a^{(s)}$'s and the single
correction term costs exactly $n$ evaluations of $\Psi$ (one per
coordinate $l = 0,\dots,n-1$, counting $\Psi(0)$ once).
\end{proof}

\Cref{alg:fast-transform} summarizes the algorithm underlying
\Cref{prop:fast-fft}.

\begin{algorithm}[t]
\caption{Fast elementwise lattice transform $Y = \Psi(X)^\top v$ for a
fixed-order power-form generator}
\label{alg:fast-transform}
\begin{algorithmic}[1]
\Require prime $n$; $t \in \mathbb F_n^\times$ of order $m \mid n-1$;
elementwise map $\Psi : \mathbb R \to \mathbb R$; data vector $v \in
\mathbb R^n$
\Ensure $Y \in \mathbb R^m$ with $Y_j = \sum_{l=0}^{n-1} v_l\, \Psi(\{l\,
t^{j-1}/n\})$
\State $Y \gets \Psi(0)\, v_0 \cdot \mathbf 1 \in \mathbb R^m$
\Comment{additive correction, $O(m)$}
\State Partition $\mathbb F_n^\times$ into the $q = (n-1)/m$ cosets of $H
= \langle t\rangle$; let $c_1,\dots,c_q$ be coset representatives
\For{$s = 1,\dots,q$}
    \State $a_i \gets \Psi\big((c_s t^{-i} \bmod n)/n\big)$, \quad $b_i
    \gets v_{c_s t^{-i} \bmod n}$, \quad for $i = 0,\dots,m-1$
    \State $r \gets \mathrm{IFFT}\big(\overline{\mathrm{FFT}(a)} \cdot
    \mathrm{FFT}(b)\big)$ \Comment{length-$m$ cyclic correlation, $O(m\log
    m)$}
    \For{$j = 1,\dots,m$}
        \State $Y_j \gets Y_j + r_{j-1}$
    \EndFor
\EndFor
\State \Return $Y$
\end{algorithmic}
\end{algorithm}

\subsection{The advantage of the fast transform}
\label{sec:when-wins}

\Cref{prop:fast-fft} is a genuine improvement over the $O(nd)$
direct-evaluation baseline, i.e.,  the cost of applying the same elementwise
feature map on top of a Halton sequence, a Sobol' sequence, or a rank-1
lattice with a generic generating vector, none of which admit a faster
elementwise transform to our knowledge. The saving, $O(nd) \to O(n\log
m)$ with $m = \Theta(d)$, comes entirely from the fixed, $n$-independent
multiplicative order of the power-form generator's scalar parameter $t$:
this is what lets $\mathbb F_n^\times$ be decomposed into $q = (n-1)/m$
cosets of the order-$m$ subgroup $H = \langle t \rangle$, each reducing to
a single length-$m$ FFT pair (\Cref{lem:coset-decomp}). It is specific to
this structured construction and gives no general recipe for accelerating
the elementwise transform for an arbitrary generating vector, a question
we do not address here. \Cref{tab:complexity} summarizes the resulting
picture, and \Cref{sec:fast-experiments} reports real wall-clock
measurements confirming it.

\begin{table}[t]
\centering
\caption{Cost of computing $Y = \Psi(X)^\top v$ for the full $n$-point
lattice, for an arbitrary elementwise map $\Psi$.}
\label{tab:complexity}
\small
\begin{tabular}{@{}p{4.6cm}p{4.0cm}p{4.5cm}@{}}
\toprule
Method & Cost & Requires \\
\midrule
Direct evaluation, i.e.,  the standard cost for a Halton or Sobol'
sequence, or a generic rank-1 lattice & $O(nd)$ & nothing \\[2pt]
\Cref{prop:fast-fft} (coset FFTs) & $O(n\log m)$\newline (i.e.\
$O(n\log d)$ for $m=\Theta(d)$) & power-form generator, small order
($m \ll n$) \\
\bottomrule
\end{tabular}
\end{table}

The comparison above is stated for $Y = \Psi(X)^\top v$ for a general
elementwise $\Psi$ from the outset, since that is this paper's object of
study: each method's $\Psi$-dependence is confined to an additive $O(n)$
(direct evaluation: $O(nm)$) term of pointwise evaluations, which changes
none of the arithmetic comparison in the table.

\subsection{The transpose direction: fast computation of $\Psi(X)w$}
\label{sec:fast-transpose}

\Cref{prop:fast-fft} computes the ``many-to-few'' aggregation $Y =
\Psi(X)^\top v \in \mathbb R^m$ from $n$ data values $v \in \mathbb R^n$.
The dual direction is a ``few-to-many'' expansion: given a coefficient
vector $w \in \mathbb R^m$ (for instance a set of per-coordinate weights,
or a linear functional to be evaluated at every lattice point), compute
$u := \Psi(X) w \in \mathbb R^n$, i.e.
\[
u_l = \sum_{j=1}^m \Psi(X)_{l,j}\, w_j = \sum_{j=1}^m \Psi\big(\{l\,
t^{j-1}/n\}\big)\, w_j, \qquad l = 0,\dots,n-1.
\]
Direct evaluation again costs $O(nm)$ arithmetic operations and $O(nm)$
evaluations of $\Psi$. The same coset structure gives an $O(n\log m)$
algorithm (plus $O(n)$ evaluations of $\Psi$) for this direction as well.

\begin{proposition}[Fast elementwise transpose lattice transform]
\label{prop:fast-Xw}
In the setting of \Cref{sec:fast-setting} (prime $n$, $t$ of order $m \mid
n-1$, $q = (n-1)/m$), for any $\Psi : \mathbb R \to \mathbb R$ and any $w
\in \mathbb R^m$, $u = \Psi(X)w$ can be computed exactly in $O(n\log m)$
arithmetic operations plus $n$ evaluations of $\Psi$, $O(n\log d)$
arithmetic operations once $m = \Theta(d)$,  as against $O(nm)$
arithmetic operations plus $nm$ evaluations of $\Psi$ for direct
evaluation.
\end{proposition}

\begin{proof}
Since $X_{0,j} = \{0\} = 0$ for every $j$, $u_0 = \sum_{j=1}^m
\Psi(0)\,w_j = \Psi(0)\sum_{j=1}^m w_j$, computed once in $O(m)$. Fix a
nonzero $l \in \mathbb F_n^\times$; it lies in a unique coset $C = \{c,
ct^{-1}, \dots, ct^{-(m-1)}\}$ of $H = \langle t\rangle$ (coset
representative $c$), so $l = ct^{-i}$ for a unique $i \in
\{0,\dots,m-1\}$. Define, for the coset containing $l$,
\[
\phi_k := \Psi\big(\{c\, t^k/n\}\big), \qquad k = 0,\dots,m-1
\]
(indices cyclic mod $m$, since $t^m = 1$). Then, using $l\, t^{j-1} = c\,
t^{j-1-i}$,
\[
u_l = \sum_{j=1}^m \Psi\big(\{c\, t^{j-1-i}/n\}\big)\, w_j =
\sum_{k=0}^{m-1} \phi_{(k-i)\bmod m}\, w_{k+1} = \corr(\phi, w)[i],
\]
the length-$m$ cyclic correlation of $\phi$ and $w$ evaluated at shift
$i$, i.e.,  the same correlation identity underlying \eqref{eq:coset-sum}, run
with the roles of the two vectors, and of ``coordinate index'' versus
``point index,'' exchanged. As in the proof of \Cref{prop:fast-fft}, each
coset's correlation here directly supplies the $m$ entries of $u$
indexed by that coset, with no summation across cosets (because each of
the $n$ entries of $u$ depends only on its own coset), and all $m$
outputs of each of the $q$ correlations are needed, since together they
supply all $n-1$ nonzero entries of $u$. Since $w \in \mathbb R^m$
already has the full length $m$ that the correlation needs, no
zero-padding is required (nor, for a general $\Psi$ with $\Psi(0) \ne
0$, would zero-padding even be meaningful).

Precompute $\mathrm{FFT}(w)$ once, in $O(m\log m)$. For each of the $q$
cosets, form $\phi^{(s)}$ ($m$ evaluations of $\Psi$), compute
$\corr(\phi^{(s)}, w) = \mathrm{IFFT}\big(\overline{
\mathrm{FFT}(\phi^{(s)})}\cdot \mathrm{FFT}(w)\big)$ in $O(m\log m)$, and
scatter its $m$ entries into the corresponding $m$ entries of $u$ in
$O(m)$. The total arithmetic cost is $O(m\log m) + O(q\cdot m\log m) =
O\big((n-1)\log m\big) + O(m\log m) = O(n\log m)$, which is $O(n\log d)$
once $m = \Theta(d)$, plus the $n-1$ evaluations of $\Psi$ used to form
the $\phi^{(s)}$'s (and the single $\Psi(0)$ for $u_0$), $n$ in total.
\end{proof}

\Cref{alg:fast-Xw} summarizes the algorithm underlying
\Cref{prop:fast-Xw}.

\begin{algorithm}[t]
\caption{Fast elementwise transpose lattice transform $u = \Psi(X)w$ for
a fixed-order power-form generator}
\label{alg:fast-Xw}
\begin{algorithmic}[1]
\Require prime $n$; $t \in \mathbb F_n^\times$ of order $m \mid n-1$;
elementwise map $\Psi : \mathbb R \to \mathbb R$; coefficient vector $w
\in \mathbb R^m$
\Ensure $u \in \mathbb R^n$ with $u_l = \sum_{j=1}^m \Psi(\{l\,
t^{j-1}/n\})\, w_j$
\State $u_0 \gets \Psi(0)\sum_{j=1}^m w_j$ \Comment{additive correction,
$O(m)$}
\State Precompute $F \gets \mathrm{FFT}(w)$
\State Partition $\mathbb F_n^\times$ into the $q = (n-1)/m$ cosets of $H
= \langle t\rangle$; let $c_1,\dots,c_q$ be coset representatives
\For{$s = 1,\dots,q$}
    \State $\phi_k \gets \Psi\big((c_s t^k \bmod n)/n\big)$, for $k =
    0,\dots,m-1$
    \State $r \gets \mathrm{IFFT}\big(\overline{\mathrm{FFT}(\phi)}\cdot
    F\big)$ \Comment{length-$m$ cyclic correlation, $O(m\log m)$}
    \For{$i = 0,\dots,m-1$}
        \State $u_{c_s t^{-i} \bmod n} \gets r_i$
    \EndFor
\EndFor
\State \Return $u$
\end{algorithmic}
\end{algorithm}

\begin{remark}[Duality]
\label{rem:duality}
\Cref{prop:fast-fft} and \Cref{prop:fast-Xw} are transpose computations, 
$Y = \Psi(X)^\top v$ and $u = \Psi(X)w$,  built from the same
coset-correlation primitive, run in complementary directions: aggregation
across many points into few coordinates (sum over cosets, plus the
$\Psi(0)v_0$ correction) versus expansion from few coordinates into many
points (no sum across cosets, plus the $\Psi(0)\sum_j w_j$ correction).
Both cost $O(n\log m) = O(n\log d)$ arithmetic operations plus $O(n)$
evaluations of $\Psi$, for $m = \Theta(d)$. Neither proof uses
continuity, monotonicity, boundedness, or $\Psi(0) = 0$: the entire
saving is a reindexing of which entries get multiplied together, so it is
exactly as insensitive to the scalar function $\Psi$ sitting at each
entry as it is to the values of $v$ or $w$ themselves; $\Psi$ enters the
cost only through the $n$ pointwise evaluations needed to form the
relevant entries of $\Psi(X)$, and the single additive correction term
that appears once $\Psi(0) \ne 0$. \Cref{prop:fast-fft} and
\Cref{prop:fast-Xw} together give a single matching pair of fast
primitives for the two directions in which the lattice-point matrix $X$,
or its elementwise transform $\Psi(X)$, is used throughout quasi-Monte
Carlo computation,  e.g., evaluating a linear functional of $m$
coefficients at all $n$ points, and aggregating $n$ function values or
residuals back into $m$ per-coordinate summaries,  exactly the operation
needed to evaluate a quasi-Monte Carlo feature map \citep{avron2016quasi}.
We emphasize that this section is purely computational: it says nothing
about how well $\Psi(X)^\top v$ or
$\Psi(X)w$ approximates any target quantity for a particular choice of
$\Psi$ (e.g.\ a random Fourier feature map), which is a separate,
approximation-theoretic question outside the scope of this paper.
\end{remark}

\subsection{Numerical experiments}
\label{sec:fast-experiments}

We verify \Cref{prop:fast-fft} and \Cref{prop:fast-Xw}
(\Cref{alg:fast-transform} and \Cref{alg:fast-Xw}) against direct $O(nm)$
evaluation of $Y = \Psi(X)^\top v$ and $u = \Psi(X)w$, i.e.,  the cost incurred
by the same elementwise feature map on top of a Halton sequence, a Sobol'
sequence, or a rank-1 lattice with a generic generating vector,  for two
choices of
$\Psi$ that stress the claimed generality from opposite ends: a
genuinely nonlinear $\Psi$ with $\Psi(0) \ne 0$, i.e., $\Psi(x) = \cos(2\pi x)
+ 0.3x^2$, in the spirit of a quasi-Monte Carlo feature map,  and the
trivial map $\Psi = \mathrm{id}$, for which the additive correction term
in every proposition above vanishes identically. 

\paragraph{Correctness.} \Cref{tab:correctness} reports, for five
$(n,m,d)$ instances spanning $n$ from $113$ to $11{,}317$ and $m$ from $7$
to $23$ (with $t$ chosen of exact order $m$ in each case) and for both
choices of $\Psi$, the maximum absolute entrywise disagreement between
\Cref{alg:fast-transform} (resp.\ \Cref{alg:fast-Xw}) and direct
evaluation of $Y$ (resp.\ $u$). Entries of $v$
and $w$ were drawn i.i.d.\ $\mathcal N(0,1)$.

\begin{table}[t]
\centering
\caption{Correctness check: maximum absolute entrywise error between
each fast method and direct $O(nm)$ evaluation, for both the nonlinear
and the trivial $\Psi$. All errors are at floating-point rounding level.}
\label{tab:correctness}
\small
\begin{tabular}{@{}rrrlrr@{}}
\toprule
$n$ & $m$ & $d$ & $\Psi$ & $\max|\text{err }Y|$ (Alg.~\ref{alg:fast-transform})
& $\max|\text{err }u|$ (Alg.~\ref{alg:fast-Xw}) \\
\midrule
113 & 7 & 5 & $\mathrm{id}$ & $2.0\times10^{-15}$ & $8.9\times10^{-16}$ \\
113 & 7 & 5 & $\cos+0.3x^2$ & $3.6\times10^{-15}$ & $1.8\times10^{-15}$ \\
1321 & 11 & 8 & $\mathrm{id}$ & $3.6\times10^{-14}$ & $1.8\times10^{-15}$ \\
1321 & 11 & 8 & $\cos+0.3x^2$ & $5.3\times10^{-14}$ & $2.7\times10^{-15}$ \\
2029 & 13 & 9 & $\mathrm{id}$ & $6.4\times10^{-14}$ & $2.7\times10^{-15}$ \\
2029 & 13 & 9 & $\cos+0.3x^2$ & $4.3\times10^{-14}$ & $1.8\times10^{-15}$ \\
919 & 17 & 12 & $\mathrm{id}$ & $1.6\times10^{-14}$ & $1.8\times10^{-15}$ \\
919 & 17 & 12 & $\cos+0.3x^2$ & $2.8\times10^{-14}$ & $2.7\times10^{-15}$ \\
11317 & 23 & 15 & $\mathrm{id}$ & $4.8\times10^{-13}$ & $3.6\times10^{-15}$ \\
11317 & 23 & 15 & $\cos+0.3x^2$ & $4.0\times10^{-13}$ & $5.3\times10^{-15}$ \\
\bottomrule
\end{tabular}
\end{table}

Every entry in \Cref{tab:correctness} is at the level expected from
double-precision floating-point rounding across the $O(m\log m)$-deep FFT
computation graph (worst case $4.8\times10^{-13}$, on data and outputs of
order $1$--$10^2$), uniformly across both choices of $\Psi$ and across both
fast methods,  direct numerical confirmation of
\Cref{prop:fast-fft} and \Cref{prop:fast-Xw} beyond the proofs above,
including the additive correction term that each proposition introduces
for a general $\Psi$ with $\Psi(0)\ne0$.

\paragraph{Timing.} \Cref{tab:timing} reports wall-clock time (best of
three runs, single-threaded \texttt{numpy}, vectorized coset construction)
for direct evaluation, i.e.,  the $O(nd)$ cost of applying the same elementwise
feature map on top of a Halton sequence, a Sobol' sequence, or a rank-1
lattice with a generic generating vector,  against
\Cref{alg:fast-transform}, computing $Y
= \Psi(X)^\top v$ of the subgroup rank-1 lattice
for $d = 50$ ($m = 53$, the smallest prime exceeding $d$), for $n$ ranging
from $16{,}007$ to $1{,}500{,}007$ and $\Psi = \mathrm{id}$; relative
errors (not shown) match \Cref{tab:correctness}'s floating-point level
throughout, topping out at $4.2\times10^{-14}$ at the largest $n$.

\begin{table}[t]
\centering
\caption{Wall-clock time (seconds, best of 3 runs) for $Y = \Psi(X)^\top
v$ ($d=50$, $m=53$, $\Psi=\mathrm{id}$ shown; nonlinear $\Psi$ discussed
in text), direct $O(nd)$ evaluation vs.\
coset-FFT (\Cref{alg:fast-transform}).}
\label{tab:timing}
\begin{tabular}{@{}rrrr@{}}
\toprule
$n$ & $q=(n-1)/m$ & direct $O(nd)$ & coset-FFT
$O(n\log m)$ \\
\midrule
16{,}007 & 302 & 0.0077 & 0.0014 \\
31{,}907 & 602 & 0.0174 & 0.0027 \\
57{,}559 & 1{,}086 & 0.0382 & 0.0046 \\
132{,}607 & 2{,}502 & 0.1114 & 0.0114 \\
337{,}081 & 6{,}360 & 0.3813 & 0.0292 \\
717{,}091 & 13{,}530 & 0.8562 & 0.0672 \\
1{,}500{,}007 & 28{,}302 & 1.8574 & 0.1412 \\
\bottomrule
\end{tabular}
\end{table}

The coset-FFT algorithm is the fastest method at every $n$ tested (already
by a factor of $\sim\!5.5\times$ over direct evaluation at $n=16{,}007$,
widening to $\sim\!13\times$ at $n=1{,}500{,}007$), consistent with
\Cref{tab:complexity}: $O(n\log m)$ with $m=53$ fixed grows only slightly
faster than linearly in $n$, while direct evaluation's $O(nd)=O(50n)$ term
grows linearly with a much larger constant. Repeating the same sweep
with the nonlinear $\Psi = \cos(2\pi\,\cdot\,) + 0.3(\cdot)^2$ leaves the
coset-FFT and direct-evaluation timings within the range predicted by
their respective $O(n)$ and $O(nm)$ evaluation-cost terms: at
$n=1{,}500{,}007$, coset-FFT rises from $0.1412\mathrm s$ to
$0.1765\mathrm s$ (a $25\%$ increase, consistent with its additive $O(n)$
evaluation term), while direct evaluation rises from $1.8574\mathrm s$ to
$5.7959\mathrm s$ (a $212\%$ increase, consistent with its $O(nm)=O(50n)$
evaluation term dominating once $\Psi$ itself is expensive to evaluate).
The transpose direction (\Cref{alg:fast-Xw}, computing $u = \Psi(X)w$) was
checked over the same design-rule setting ($d=50$, $m=53$) at $n \in
\{31{,}271,\ 170{,}873,\ 725{,}041,\ 1{,}517{,}921\}$: coset-FFT again wins
throughout, by a factor growing from $\sim\!9\times$ (id) /
$\sim\!20\times$ (nonlinear) at the smallest $n$ to $\sim\!17\times$ /
$\sim\!35\times$ at the largest, with relative errors again at
floating-point level ($\le 6.9\times10^{-16}$ throughout,  smaller than
the $Y$-direction errors, since $u$'s coset-FFT computation involves one
correlation and a scatter rather than a cross-coset summation).

\section{Convergence rate for a fixed subgroup-order generator}
\label{sec:rate}

\Cref{sec:fast} shows that restricting $t$ to have small, fixed
multiplicative order $m$ is computationally attractive. This section asks
whether it is a statistically sound restriction at all: if the generating
vector is built from a subgroup order that is fixed independently of $n$,
does the resulting \emph{full} $n$-point lattice rule $Q_n$ still converge
as $n \to \infty$? This is not automatic. As $n$ grows, $t$ ranges over
only the $\varphi(m)$ generators of a single fixed-order subgroup,  a
pool that does not grow with $n$,  so the generating vector is far from
the ``generic, CBC-optimal'' case that the classical convergence theory
for lattice rules is built around; that theory typically argues by
averaging the worst-case error over a pool of candidate generators whose
size grows with $n$ (e.g.\ all of $\mathbb F_n^\times$), which is simply
unavailable here.

\subsection{Setting}
\label{sec:rate-setting}

In this section $m$ is additionally assumed \emph{prime}. Fix $d$ with $2
\le d \le m-1$. For every prime $n \equiv 1 \pmod m$ and every $t \in
\mathbb F_n^\times$ of multiplicative order exactly $m$ (there are
$\varphi(m) = m-1$ such $t$, since $m$ is prime), let $\mathbf z(t) :=
(1,t,\dots,t^{d-1}) \bmod n$ and let $Q_n$ be the \emph{full} $n$-point
lattice rule of \Cref{sec:rank1} built from $\mathbf z(t)$.

\subsection{Algebraic tools: cyclotomic polynomials and resultants}

\begin{lemma}[Order-$m$ elements are roots of the cyclotomic polynomial
modulo $n$]
\label{lem:cyclotomic}
Let $\Phi_m(x) := 1 + x + \cdots + x^{m-1} \in \mathbb Z[x]$ ($m$ prime;
the $m$-th cyclotomic polynomial, irreducible over $\mathbb Q$ and of
degree $m-1$; see e.g.\ \citealt[Ch.~7]{irelandrosen1990classical}). If $n$
is prime, $m \mid n-1$, and $t \in \mathbb F_n^\times$ has order exactly
$m$, then $\Phi_m(t) \equiv 0 \pmod n$, and moreover $\Phi_m$ splits into
$m-1$ distinct linear factors over $\mathbb F_n$, one root per order-$m$
element of $\mathbb F_n^\times$.
\end{lemma}

\begin{proof}
Over $\mathbb Z$, $x^m - 1 = (x-1)\Phi_m(x)$. Reducing mod $n$: $(t-1)
\Phi_m(t) \equiv t^m - 1 \equiv 0 \pmod n$, since $t^m = 1$. As $n$ is
prime, $\mathbb F_n$ is a field; since $t$ has order exactly $m > 1$, $t
\not\equiv 1$, so $t-1$ is invertible mod $n$, forcing $\Phi_m(t) \equiv 0
\pmod n$. For the splitting claim: $m \mid n-1$ means the cyclic group
$\mathbb F_n^\times$ of order $n-1$ contains a full set of $m$-th roots of
unity, so $x^m - 1$ splits into $m$ distinct linear factors over
$\mathbb F_n$ (distinct because $\gcd(m,n) = 1$ makes $x^m-1$ separable);
removing the factor $(x-1)$ leaves $\Phi_m$ splitting into the remaining
$m-1$, which are exactly the elements of order dividing $m$ but $\ne 1$,
i.e.\ order exactly $m$ since $m$ is prime.
\end{proof}

\begin{lemma}[Resultant criterion]
\label{lem:resultant}
For $h = (h_0,\dots,h_{d-1}) \in \mathbb Z^d \setminus \{\mathbf 0\}$
write $P_h(x) := \sum_{j=0}^{d-1} h_j x^j$ and $R(h) := \Res(P_h,\Phi_m)
\in \mathbb Z$ (the Sylvester resultant, taken at formal degrees $d-1$ and
$m-1$). Then:
\begin{enumerate}
\item[(a)] $R(h) \ne 0$ (uses $d \le m-1$).
\item[(b)] $|R(h)| \le (d\, \|h\|_\infty)^{m-1}$.
\item[(c)] If $h \cdot \mathbf z(t) \equiv 0 \pmod n$ then $n \mid R(h)$.
\end{enumerate}
\end{lemma}

\begin{proof}
Since $\Phi_m$ is monic, the resultant identity $\Res(f,\Phi_m) =
\prod_{\Phi_m(\beta)=0} f(\beta)$ holds for any $f$ of formal degree $\le
m-2$, the product over the $m-1$ roots of $\Phi_m$ in $\mathbb C$
(equivalently, in $\mathbb F_n$ once $m \mid n-1$, by
\Cref{lem:cyclotomic}); this identity needs only that $\Phi_m$'s leading
coefficient ($=1$) does not vanish, not any assumption on $f$'s actual
degree. Apply it over $\mathbb C$ with $\zeta := e^{2\pi i/m}$: $R(h) =
\prod_{k=1}^{m-1} P_h(\zeta^k)$.

\emph{(a)} If $R(h) = 0$ then $P_h(\zeta^k) = 0$ for some $k \in
\{1,\dots,m-1\}$, i.e.\ $P_h$ has the primitive $m$-th root of unity
$\zeta^k$ as a root. Since $\Phi_m$ is the minimal polynomial of $\zeta^k$
over $\mathbb Q$, $\Phi_m \mid P_h$ in $\mathbb Q[x]$, forcing $\deg P_h
\ge m-1$ or $P_h \equiv 0$. Both are excluded: $\deg P_h \le d - 1 \le
m-2$ and $h \ne \mathbf 0$. So $R(h) \ne 0$.

\emph{(b)} $|P_h(\zeta^k)| \le \sum_{j=0}^{d-1} |h_j| \le d\|h\|_\infty$
for every $k$ (triangle inequality, $|\zeta^{jk}| = 1$); multiply the
$m-1$ factors.

\emph{(c)} Applying the same resultant-as-product formula over
$\mathbb F_n$ (valid since $\Phi_m$ splits there, by
\Cref{lem:cyclotomic}): $\overline{R(h)} = \prod_k \overline{P_h}(\beta_k)$
over the $m-1$ roots $\beta_k \in \mathbb F_n^\times$ of $\Phi_m$ mod $n$,
one of which is $t$. Since $h \cdot \mathbf z(t) \equiv 0 \pmod n$ means
exactly $\overline{P_h}(t) \equiv 0 \pmod n$, that factor vanishes, so the
whole product is $0$ in $\mathbb F_n$, i.e.\ $n \mid R(h)$ (reduction mod
$n$ of the integer $R(h)$, an integer polynomial in the coefficients of
$P_h$ and $\Phi_m$, via the Sylvester determinant,  agrees with the
resultant of the reduced polynomials, since reduction mod $n$ is a ring
homomorphism and the determinant is a fixed polynomial in the matrix
entries).
\end{proof}

\begin{lemma}[Tail bound for the Korobov weight]
\label{lem:tail}
For $\alpha > 1$, $d \ge 1$, $T \ge 1$,
\begin{equation}
    \sum_{\substack{h \in \mathbb Z^d \\ \|h\|_\infty > T}} \rho(h)^{-1}
\;\le\; \kappa(d,\alpha)\, T^{-(\alpha-1)}, \qquad \kappa(d,\alpha) :=
\frac{2d}{\alpha-1} \big(1 + 2\zeta(\alpha)\big)^{d-1}.
\end{equation}
where $\zeta(\alpha)$ denotes the Riemann Zeta function.
\end{lemma}

\begin{proof}
Union-bound over which coordinate exceeds $T$ in absolute value:
\[
\sum_{\|h\|_\infty > T} \rho(h)^{-1} \le \sum_{i=1}^d \Big(
\sum_{|k|>T} \max(1,|k|)^{-\alpha} \Big) \prod_{j\ne i} \Big(
\sum_{k\in\mathbb Z} \max(1,|k|)^{-\alpha} \Big) = d\cdot S_T \cdot
\big(1+2\zeta(\alpha)\big)^{d-1},
\]
where $S_T := 2\sum_{k>T} k^{-\alpha} \le 2\int_T^\infty x^{-\alpha}\,dx =
\frac{2}{\alpha-1} T^{1-\alpha}$ by integral comparison (valid for $T \ge
1$, $\alpha>1$).
\end{proof}

\subsection{Main convergence theorem}

\begin{theorem}[Convergence rate for a fixed prime subgroup order]
\label{thm:rate}
Let $m$ be prime, $2 \le d \le m-1$, $\alpha > 1$. Set
\[
n_0(d,m) := (4d)^{m-1}, \qquad C(d,\alpha,m) := \kappa(d,\alpha)\,
(4d)^{\alpha-1} = \frac{2d(4d)^{\alpha-1}}{\alpha-1} \big(1 +
2\zeta(\alpha)\big)^{d-1}.
\]
Then for every prime $n \ge n_0(d,m)$ with $n \equiv 1 \pmod m$ and every
$t \in \mathbb F_n^\times$ of order exactly $m$,
\begin{equation}
e^2(Q_n, \mathbf z(t)) \;\le\; C(d,\alpha,m)\; n^{-(\alpha-1)/(m-1)}.
\label{eq:rate}
\end{equation}
Equivalently, $e(Q_n,\mathbf z(t)) = O_{d,\alpha,m}\big(n^{-(\alpha-1)/(2(m-1))}\big)$
as $n \to \infty$ along primes $n \equiv 1 \pmod m$.
\end{theorem}

\begin{proof}
Let $T_n := \lfloor n^{1/(m-1)}/(2d) \rfloor$. For $n \ge n_0(d,m)$, $T_n
\ge n^{1/(m-1)}/(4d)$ (the floor loses at most a factor $2$ once
$n^{1/(m-1)}/(2d) \ge 2$, which holds exactly when $n \ge n_0(d,m)$). By
\Cref{lem:resultant}(b), $(d\,T_n)^{m-1} \le \big(n^{1/(m-1)}/2\big)^{m-1}
= n/2^{m-1} < n$. So if some nonzero $h$ with $\|h\|_\infty \le T_n$
satisfied $h\cdot\mathbf z(t) \equiv 0 \pmod n$, \Cref{lem:resultant}(c)
would give $n \mid R(h)$, and $R(h) \ne 0$ by \Cref{lem:resultant}(a)
(using $d \le m-1$), forcing $n \le |R(h)| \le (dT_n)^{m-1} < n$, a
contradiction. Hence every nonzero $h$ with $h \cdot \mathbf z(t) \equiv 0
\pmod n$ has $\|h\|_\infty > T_n$. By \eqref{eq:aliasing},
\[
e^2(Q_n, \mathbf z(t)) = \!\!\!\sum_{\substack{h \ne 0 \\ h\cdot z(t)
\equiv 0\,(n)}}\!\!\! \rho(h)^{-1} \;\le\!\! \sum_{\|h\|_\infty > T_n}
\!\!\rho(h)^{-1} \overset{\text{Lem.\ \ref{lem:tail}}}{\le}
\kappa(d,\alpha)\, T_n^{-(\alpha-1)} \le \kappa(d,\alpha)\,
(4d)^{\alpha-1}\, n^{-(\alpha-1)/(m-1)},
\]
using $T_n \ge n^{1/(m-1)}/(4d)$ in the last step. This is
\eqref{eq:rate}.
\end{proof}

\begin{figure}[t]
\centering
{
\includegraphics[width=0.8\linewidth]{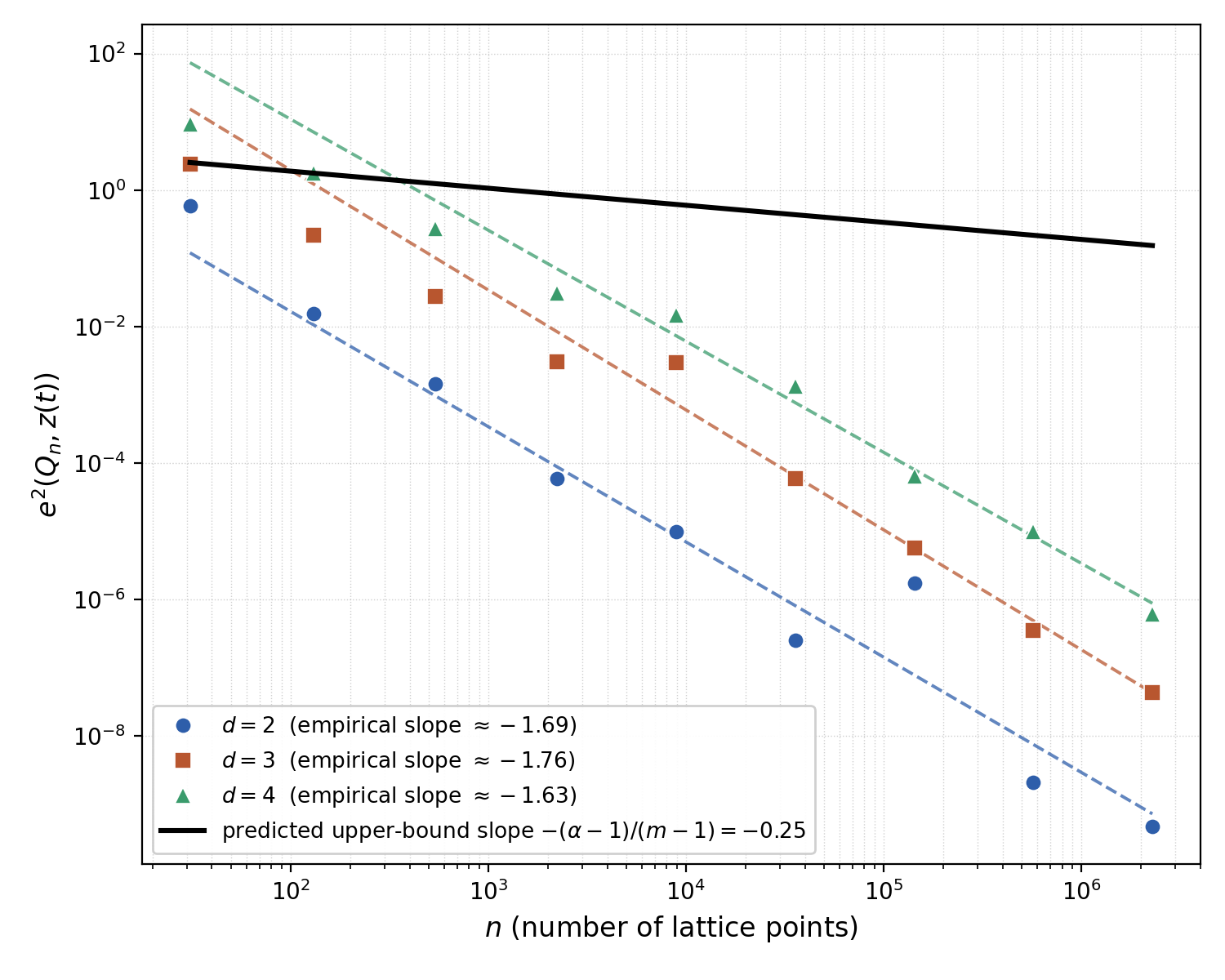}}
\caption{empirical decay v.s. predicted upper-bound exponent at $m=5$ (prime) and $\alpha=2$}
\label{Rate}
\end{figure}

\subsection{Sharpness of the dimension restriction}

\begin{proposition}[A permanent floor once $d > m-1$]
\label{prop:floor}
If $d \ge m$, then for every $n,t$ as in \Cref{sec:rate-setting},
\[
e^2(Q_n, \mathbf z(t)) \;\ge\; 2,
\]
so $e^2(Q_n,\mathbf z(t)) \not\to 0$ for every $n$ in the family,
regardless of how large $n$ grows.
\end{proposition}

\begin{proof}
Let $h^\star := (1,1,\dots,1,0,\dots,0) \in \mathbb Z^d$ be the
coefficient vector of $\Phi_m$ padded with zeros to length $d$ (so
$\rho(h^\star)^{-1} = 1$, since every nonzero entry equals $1$). By
\Cref{lem:cyclotomic}, $\Phi_m(t) \equiv 0 \pmod n$ identically, i.e.\
$h^\star \cdot \mathbf z(t) = 1 + t + \cdots + t^{m-1} \equiv 0 \pmod n$
for \emph{every} valid pair $(n,t)$, not merely below some threshold.
Since $h^\star \ne \mathbf 0$ (as $d \ge m > 0$), it contributes
$\rho(h^\star)^{-1} = 1$ to the sum \eqref{eq:aliasing}; so does $-h^\star$
(same congruence, same weight, since $\rho(-h) = \rho(h)$), and $h^\star
\ne -h^\star$. Both terms are always present, giving the stated floor.
\end{proof}

\begin{corollary}[Exact threshold]
\label{cor:threshold}
For a fixed-order subgroup generator with $m$ prime, the full lattice
$Q_n$ converges, at the explicit rate \eqref{eq:rate}, if $m \ge d+1$, and
provably fails to converge,  with an explicit, $n$-independent floor
$e^2(Q_n, \mathbf z(t)) \ge 2$,  if $m \le d$. The condition ``$m$ prime, $m \ge d+1$'' is
therefore not merely a safe sufficient margin; it is the exact threshold,
one unit inside the boundary $m = d$ at which convergence becomes
impossible.
\end{corollary}

\section{A sharper constant via averaging over the candidate pool}
\label{sec:averaging}

\Cref{thm:rate} is proved \emph{uniformly}: the argument never distinguishes
among the $m-1$ order-$m$ generators $t_1,\dots,t_{m-1} \in
\mathbb F_n^\times$, so the same bound \eqref{eq:rate} holds for every one
of them simultaneously. This raises a natural question,  can \emph{averaging} the
worst-case error over these $m-1$ candidates certify a sharper bound for
at least one of them? A
naive combination of the tools already in hand does not work: bounding, by
Lagrange's theorem, how many of the $m-1$ candidates can share a single
bad $h$ by $d-1$, independent of $\|h\|_\infty$, gives a multiplicity
bound that cannot be combined with \Cref{thm:rate}'s size-dependent
threshold, and yields no decay in $n$ at all. What does work is a sharper
fact from algebraic number theory, specific to the cyclotomic structure
already underlying \Cref{lem:cyclotomic}.

\subsection{Complete splitting and the resultant as a field norm}

\begin{lemma}[Complete splitting; the resultant as a field norm]
\label{lem:splitting}
Let $\zeta_m := e^{2\pi i/m}$, $K := \mathbb Q(\zeta_m)$, and $\OK =
\mathbb Z[\zeta_m]$ its ring of integers ($m$ prime). Since $m \mid n-1$,
the rational prime $n$ splits completely in $\OK$,
\[
n\OK = \mathfrak p_1 \mathfrak p_2 \cdots \mathfrak p_{m-1},
\]
a product of $m-1$ distinct prime ideals, each of norm exactly $n$, in
canonical bijection with the $m-1$ order-$m$ elements $t_1,\dots,t_{m-1}
\in \mathbb F_n^\times$: $\OK/\mathfrak p_i \cong \mathbb F_n$ via $\zeta_m
\mapsto t_i$, so for $\beta \in \OK$, $\beta \equiv 0 \pmod{\mathfrak
p_i}$ iff the image of $\beta$ under $\zeta_m \mapsto t_i$ vanishes mod
$n$. Moreover, for $h \in \mathbb Z^d \setminus \{\mathbf 0\}$, $d \le
m-1$, setting $\beta := P_h(\zeta_m) \in \OK$,
\[
R(h) = N_{K/\mathbb Q}(\beta) := \prod_{\sigma \in \Gal(K/\mathbb Q)}
\sigma(\beta).
\]
\end{lemma}

\begin{proof}
That $n$ splits completely in $\mathbb Q(\zeta_m)$ iff $n \equiv 1 \pmod
m$ is classical (see e.g.\ \citealt[Thm.~2.13]{washington1997introduction},
or Dedekind's factorization theorem applied to $\Phi_m$, whose reduction
mod $n$ we already showed in \Cref{lem:cyclotomic} splits into $m-1$
distinct linear factors,  one irreducible factor of degree $1$ per prime
above $n$, since $n$ is unramified, $\gcd(n,m) = 1$). The
Kummer--Dedekind correspondence identifies each prime $\mathfrak p_i$ with
the ideal $(n, \zeta_m - t_i)$, giving the stated residue map. For the
norm identity: $\Gal(K/\mathbb Q) \cong (\mathbb Z/m\mathbb Z)^\times$
acts via $\sigma_k : \zeta_m \mapsto \zeta_m^k$ for $k = 1,\dots,m-1$ (all
of $(\mathbb Z/m\mathbb Z)^\times$ since $m$ is prime), and $\sigma_k(\beta)
= \sigma_k(P_h(\zeta_m)) = P_h(\zeta_m^k)$ since $P_h$ has rational
coefficients. So $N_{K/\mathbb Q}(\beta) = \prod_{k=1}^{m-1}
P_h(\zeta_m^k)$, matching the product formula already used to define
$R(h)$ in \Cref{lem:resultant}.
\end{proof}

\begin{lemma}[Multiplicity forces a higher power of $n$]
\label{lem:multiplicity}
For $h \in \mathbb Z^d \setminus \{\mathbf 0\}$, $d \le m-1$, define
\[
N_h := \#\{ i \in \{1,\dots,m-1\} : h \cdot \mathbf z(t_i) \equiv 0
\pmod n \}.
\]
Then $n^{N_h} \mid R(h)$.
\end{lemma}

\begin{proof}
Let $\beta := P_h(\zeta_m)$ as in \Cref{lem:splitting}. By construction,
$h \cdot \mathbf z(t_i) \equiv 0 \pmod n$ iff $\beta \equiv 0 \pmod{
\mathfrak p_i}$. If this holds for $N_h$ distinct indices, $\beta$ lies in
each of $N_h$ distinct (hence pairwise coprime) prime ideals of $\OK$; by
unique factorization of ideals in the Dedekind domain $\OK$, their
product, an ideal of norm $n^{N_h}$,  divides the ideal $(\beta)$.
Norms are multiplicative over ideal divisibility, so $n^{N_h} \mid
N_{K/\mathbb Q}(\beta) = R(h)$ (in absolute value).
\end{proof}

This is a genuine strengthening of \Cref{lem:resultant}(c), which only
gives $n^1 \mid R(h)$: the more of the $m-1$ candidates share a bad $h$,
the more of $n$'s ``budget'' inside $R(h)$'s bounded size gets used up.

\Cref{tab:multiplicity} illustrates \Cref{lem:multiplicity} at $m=5$,
$d=3$: for each prime $n$, we searched directly for the shortest $h$ that
is simultaneously bad ($h\cdot\mathbf z(t_i)\equiv 0\pmod n$) for two
order-$5$ elements $t_i$ that are \emph{not} an inverse pair. (An inverse
pair $t,t^{-1}$ admits an unrelated, cheaper shortcut, replacing $t$ by
$t^{-1}$ only reverses the $d$ coordinates of $\mathbf z(t)$, so
$e^2(Q_n,\mathbf z(t)) = e^2(Q_n,\mathbf z(t^{-1}))$ exactly, making a
shared bad $h$ far easier to arrange between such a pair than between two
unrelated candidates.) We verified independently, by direct computation
of both the order-$5$ elements mod $n$ and the resultant $R(h)$, that
$R(h)/n^2$ is an exact integer in every row, confirming $n^2 \mid R(h)$
(not merely $n^1$) whenever $N_h = 2$, exactly as \Cref{lem:multiplicity}
predicts.

\begin{table}[t]
\centering
\caption{Verification of \Cref{lem:multiplicity} at $m=5$, $d=3$: the
shortest $h$ simultaneously bad for two non-inverse order-$5$ elements,
and the resulting resultant.}
\label{tab:multiplicity}
\begin{tabular}{@{}rrrr@{}}
\toprule
$n$ & shortest simultaneous-bad $h$ & $R(h)$ & $R(h)/n^2$ \\
\midrule
$41$ & $(-7,2,-3)$ & $1{,}681$ & $1$ \\
$131$ & $(-7,2,-13)$ & $17{,}161$ & $1$ \\
$401$ & $(-27,15,-19)$ & $160{,}801$ & $1$ \\
$1231$ & $(43,-32,-84)$ & $92{,}437{,}021$ & $61$ \\
\bottomrule
\end{tabular}
\end{table}

\subsection{A sharper average, via a log-weighted tail bound}

Converting \Cref{lem:multiplicity} into a bound on the average worst-case
error over the $m-1$ candidates requires a version of \Cref{lem:tail}
that also tracks the multiplicity $N_h$, which \Cref{lem:multiplicity}
controls only logarithmically (via \Cref{lem:resultant}(b)). This forces
a log-weighted tail bound.

\begin{lemma}[Log-weighted tail bound, explicit constant]
\label{lem:log-tail}
For $\alpha > 1$, $d \ge 2$, $T \ge 3$,
\begin{equation}
\sum_{\|h\|_\infty > T} \rho(h)^{-1} \ln(d\|h\|_\infty) \;\le\;
\lambda(d,\alpha)\, T^{-(\alpha-1)} \ln(dT),
\label{eq:log-tail-explicit}
\end{equation}
where
\[
\lambda(d,\alpha) := \frac{2d}{\alpha-1}\big(1+2\zeta(\alpha)\big)^{d-2}
\left[\frac{\alpha\big(1+2\zeta(\alpha)\big)}{\alpha-1} +
2(d-1)\big({-\zeta'(\alpha)}\big)\right].
\]
\end{lemma}

\begin{proof}
Using $\|h\|_\infty \le \prod_j \max(1,|h_j|)$ (each factor $\ge 1$),
$\ln(d\|h\|_\infty) \le \ln d + \sum_{j=1}^d \ln^+ |h_j|$. Bound
$\sum_{\|h\|_\infty>T} \rho(h)^{-1} \ln d$ by \Cref{lem:tail}. For each
$j$, bound $A_j := \sum_{\|h\|_\infty>T} \rho(h)^{-1} \ln^+|h_j|$ by the
union bound over the dominant coordinate exactly as in \Cref{lem:tail}'s
proof: the ``dominant coordinate is $j$ itself'' case uses the tail
estimate $2\sum_{k>T} k^{-\alpha}\ln k \le \frac{2T^{1-\alpha}}{\alpha-1}
\big(\ln T + \frac1{\alpha-1}\big)$ (integration by parts; $x^{-\alpha}\ln
x$ is eventually decreasing, valid once $T \ge 3 > e^{1/\alpha}$ for every
$\alpha>1$), contributing a factor $(1+2\zeta(\alpha))^{d-1}$ from the
other $d-1$ unweighted coordinates; the ``dominant coordinate is some $i
\ne j$'' case ($d-1$ such $i$) uses $2\sum_{k>T} k^{-\alpha} \le
\frac2{\alpha-1}T^{1-\alpha}$ (\Cref{lem:tail}) times the finite constant
$L(\alpha) := \sum_k \max(1,|k|)^{-\alpha}\ln^+|k| = 2\sum_{k\ge1}
k^{-\alpha}\ln k = -2\zeta'(\alpha)$ in place of $j$'s own factor, times
$(1+2\zeta(\alpha))^{d-2}$ from the remaining coordinates. Summing $A_j$
over $j=1,\dots,d$ (all equal by symmetry) and adding the $\ln d$ piece,
then using $\ln(dT) \ge \ln 3 > 1$ (from $T \ge 3$) to absorb every bare
additive constant, the $1/(\alpha-1)$ inside the integration-by-parts
term, and the implicit factor needed to fold the $\ln d$ piece in, into
the $\ln(dT)$-carrying term, collects everything into the stated closed
form.
\end{proof}

\subsection{Existence of a sharper-constant generator}

\begin{theorem}[Existence of a sharper-constant generator]
\label{thm:average}
Let $m$ be prime, $2 \le d \le m-1$, $\alpha > 1$, and let $t_1,\dots,
t_{m-1}$ be the $m-1$ order-$m$ elements of $\mathbb F_n^\times$. Then, as
$n \to \infty$ along primes $n \equiv 1 \pmod m$,
\begin{equation}
\frac1{m-1} \sum_{i=1}^{m-1} e^2(Q_n, \mathbf z(t_i)) \;\le\;
\frac{\lambda(d,\alpha)(4d)^{\alpha-1}}{m-1}\, n^{-(\alpha-1)/(m-1)}
\big(1 + o(1)\big).
\label{eq:average-rate}
\end{equation}
In particular, there exists $t^\star \in \{t_1,\dots,t_{m-1}\}$ (at
least one, possibly $n$-dependent) with $e^2(Q_n, \mathbf z(t^\star))$ at
most the average, a factor $\Theta(m-1)$ below \Cref{thm:rate}'s bound
\eqref{eq:rate}, which holds uniformly for \emph{every} choice.
\end{theorem}

\begin{proof}
Averaging over the $m-1$ candidates and exchanging summation order,
\[
\frac1{m-1} \sum_{i=1}^{m-1} e^2(Q_n, \mathbf z(t_i)) = \sum_{h \ne 0}
\rho(h)^{-1} \frac{N_h}{m-1}.
\]
For $\|h\|_\infty \le T_n$ (the \Cref{thm:rate} threshold), $N_h = 0$ by
\Cref{lem:resultant} exactly as before, so these terms vanish. For
$\|h\|_\infty > T_n$, \Cref{lem:multiplicity} combined with
\Cref{lem:resultant}(b) gives $n^{N_h} \le |R(h)| \le (d\|h\|_\infty)^{m-1}$,
i.e.\ $N_h \le (m-1)\ln(d\|h\|_\infty)/\ln n$. So
\[
\sum_{\|h\|_\infty > T_n} \rho(h)^{-1} \frac{N_h}{m-1} \le \frac1{\ln n}
\sum_{\|h\|_\infty > T_n} \rho(h)^{-1} \ln(d\|h\|_\infty)
\overset{\text{Lem.~\ref{lem:log-tail}}}{\le} \frac{\lambda(d,\alpha)}{\ln
n}\, T_n^{-(\alpha-1)} \ln(dT_n).
\]
With $T_n \sim n^{1/(m-1)}/(4d)$ (\Cref{thm:rate}), $T_n^{-(\alpha-1)}
\sim (4d)^{\alpha-1} n^{-(\alpha-1)/(m-1)}$ and $\ln(dT_n)/\ln n \to
1/(m-1)$, giving \eqref{eq:average-rate}.
\end{proof}

\section{Experiments}
\label{sec:experiments}

\Cref{sec:fast} and \Cref{sec:rate} make a structural argument for the
fixed-order power-form generator: it is fast to turn into an elementwise
feature map, and, despite drawing from a pool of only $\varphi(m)=m-1$
candidates rather than one that grows with $n$, it still converges at an
explicit rate. This section asks the complementary empirical question: on
the downstream task that motivates the construction in the first place,  does the
resulting feature map actually estimate the target kernel more accurately
than the constructions practitioners use today? We evaluate this on two
tasks that instantiate $\Psi(X)w$-style expansion maps directly: a
synthetic estimator of the exponential inner-product kernel
$\exp(\mathbf x^\top\mathbf y)$ (\Cref{sec:exp-synthetic}), and a
self-normalized softmax-attention estimator on nine real embedding
datasets (\Cref{sec:exp-realdata}), the latter being exactly the
large-language-model attention application named in
\Cref{sec:intro}.  Both compare our subgroup rank-1
lattice rule  against four feature-map
constructions in standard use, namely,  i.i.d.\ Gaussian random features,
Orthogonal Random Features \citep[ORF;][]{yu2016orthogonal}, and scrambled
Sobol' and Halton low-discrepancy sequences,  all mapped through the same
elementwise nonlinearity so that any difference in accuracy is attributable
to the point set alone. 

\paragraph{Method and generator selection, common to both experiments.}
For a target dimension $d$, the subgroup order is fixed at $m = $ the
smallest prime with $m \ge 2d-1$ (so $m \ge d+1$ whenever $d \ge 2$,
satisfying \Cref{def:fixed-order} and the threshold of \Cref{thm:rate}; the lattice size
$n$ is then the smallest prime with $n \equiv 1 \pmod m$ at or above a
target $n_0$. Because $m$ is prime, every nonzero residue generates the
order-$m$ subgroup $H$ (\Cref{ex:subgroup}), so all $\varphi(m)=m-1$
admissible power-form generators are available; for each $(d,m,n)$
configuration we search this entire pool exhaustively, evaluating every
candidate's empirical error on a hold-out batch via the coset-FFT
elementwise transform of \Cref{sec:fast} and caching the arg-min to disk
for reuse. This is the same pool \Cref{thm:rate} bounds \emph{uniformly}
and \Cref{thm:average} shows is, \emph{on average}, a factor $\Theta(m-1)$
below that uniform bound; the experiments below pick the best of the pool based 
on the hold-out batch rather than relying on either guarantee directly, so
the results are a best-case realization of what a $\varphi(m)$-sized
candidate pool can deliver, not a test of \Cref{thm:rate}'s bound itself.
This search is the one place a fixed subgroup order costs something beyond
the $O(n\log m)$ evaluation cost of \Cref{sec:fast}: it is $O(m)$
candidate evaluations, each itself dominated by FFT work whose cost can
vary by close to an order of magnitude between otherwise similarly sized
primes $m$, depending on which prime-length FFT algorithm the underlying
FFT implementation falls back to; in the synthetic sweep below it ranged
from under a second at small $d$ to $3.02$ hours for the single largest
configuration tested ($d=2048$, $n\approx 2\times10^7$)  on a single GPU. Because the
result depends only on $(d,m,n)$ and not on the downstream task, it needs
to be paid once and cached, after which evaluation-only runs incur no
search cost at all.

\subsection{Synthetic kernel estimation}
\label{sec:exp-synthetic}

In this subsection, we evaluate the subgroup rank-1 lattice  by comparing the approximation of the exponential inner-product kernel studied in ~\citep{choromanski2021rethinking}: 
\begin{equation}
\label{inneProductApp}
    \exp(\mathbf x^\top\mathbf y) = \exp(-\frac{\|\mathbf x \|^2 + \|\mathbf y \|^2}{2}) \mathbb{E}_{\boldsymbol{w} \sim \mathcal{N}(\boldsymbol{0},\boldsymbol{I}_d)} \text{cosh} (\boldsymbol{w}^\top (\mathbf x + \mathbf y))
\end{equation}
The approximation in (\ref{inneProductApp}) needs constructing feature maps by  samples from Gaussian distribution  $\mathcal{N}(\boldsymbol{0},\boldsymbol{I}_d)$. Thus, for the Quasi-Monte Carlo baselines (Sobol' and Halton),  we construct the Quasi Gaussian samples from transforming the QMC points in cube $[0,1)^d$ to the Gaussian via element-wise inverse  cumulative density function  of standard Gaussian, i.e., $\Phi^{-1}(\cdot)$. For subgroup rank-1 lattice, we construct Quasi Gaussian samples  via element-wise  mapping:
\begin{equation}
    \Psi(X) = \Phi^{-1}(\{X+0.5\})  ,
\end{equation}
where $\{\cdot\}$ denotes the element-wise operation that takes the  fractional part. Notable, $\Psi(0) = \Phi^{-1}(\{0+0.5\}) = 0$.

For each trial, vectors $\mathbf x,\mathbf y\in\R^d$ are drawn i.i.d.\
$\mathrm{Uniform}[0,1]^d$ and rescaled by a fixed global norm constant; the
ground truth is $\mathrm{GT}=\exp(\mathbf x^\top\mathbf y)$, each method
produces an estimate $\widehat{\mathrm{GT}}$ from $n$ feature samples, and the reported error is the
mean squared relative error $\mathrm{RLerr}=\mathrm{mean}\big((|\mathrm{GT}
-\widehat{\mathrm{GT}}|/\mathrm{GT})^2\big)$, averaged over a batch of trial
pairs and further averaged, with standard deviation, over independent
repeats. We sweep $d\in\{8,16,32,64,128,256,512,1024,2048\}$ and, for each
$d$, six target ratios $n_0/d $  such that  $n_0/d \in\{100,200,1000,2000,10000,20000\}$ with
$n_0=(n_0/d)\cdot d$, for $9\times6=54$ configurations in total; the number of repeats is fixed at $n\_run=5$ and the trial-batch size at $BN=2000$ for \emph{every} configuration,  so that all $54$ configurations  are evaluated with the same statistical power. The index \texttt{best\_jj} is reused from the cached exhaustive search over $BN\_search=512$ independent trials described above rather than recomputed.
The evaluation-only time of the full sweep was $100{,}404.1\,\mathrm s$ ($\approx27.9$ hours) on a single GPU,
with the single largest configuration ($d=2048$, $n_0/d=20{,}000$) alone accounting for $43{,}605.1\,\mathrm s$ ($\approx12.1$ hours) of that total. Almost all of this time is spent in the baselines: summed over all $54$ configurations, ORF took $88{,}581\,\mathrm s$, Halton $9{,}138\,\mathrm s$, Sobol $2{,}423\,\mathrm s$ and Gaussian $104\,\mathrm s$, whereas the subgroup rank-1 lattice took only $161\,\mathrm s$ in total.

\subsubsection{Empirical Error Evaluation}

\paragraph{Headline result.} The subgroup lattice attains the lowest mean
error in $49$ of the $54$ configurations ($91\%$); \Cref{tab:synth-exceptions}
lists every exception, and in each of them the most accurate method is one of
the two QMC baselines (Sobol or Halton).
Four of the five exceptions are concentrated at the smallest tested
dimensions ($d=8,16,32$) and only at the largest tested ratios; the fifth
($d=1024$, $n_0/d=1000$) is a near-tie ($1.5\%$ apart) discussed separately
below. For every $d\ge64$ other than that single $d=1024$ point, the
subgroup lattice is the best method at \emph{every} $n$ tested.

\begin{table}[h]
\centering
\small
\caption{Synthetic $\exp(\mathbf x^\top\mathbf y)$ estimation: the five (of
$54$) configurations in which the subgroup lattice was not the most
accurate method. Four of them are at the smallest tested dimensions and
the largest tested $n$; the fifth is an isolated near-tie at $d=1024$.}
\label{tab:synth-exceptions}
\resizebox{\textwidth}{!}{%
\begin{tabular}{@{}rrrlrr@{}}
\toprule
$d$ & $n_0/d$ & $n$ & winning method & subgroup RLerr & winner's RLerr \\
\midrule
8    & 10{,}000 & 80{,}071    & Halton & $1.083\times10^{-4}$ & $9.856\times10^{-5}$ (within noise) \\
8    & 20{,}000 & 160{,}073   & Sobol & $7.927\times10^{-5}$ & $3.455\times10^{-5}$ \\
16   & 20{,}000 & 320{,}107   & Sobol & $3.982\times10^{-5}$ & $2.428\times10^{-5}$ \\
32   & 10{,}000 & 321{,}199   & Sobol & $8.916\times10^{-5}$ & $2.895\times10^{-5}$ \\
1024 & 1{,}000  & 1{,}047{,}031 & Halton & $1.470\times10^{-5}$ & $1.448\times10^{-5}$ (within noise, $1.5\%$ apart) \\
\bottomrule
\end{tabular}
}
\end{table}

\begin{table}[t]
\centering
\small
\caption{Advantage of the subgroup lattice (best-of-\{Gaussian, ORF,
Sobol, Halton\} error $\div$ subgroup error; $>1$ favors the subgroup
lattice) at the smallest and the largest tested $n$ for each $d$, with the
closest competing baseline named in parentheses.}
\label{tab:synth-advantage}
\resizebox{\textwidth}{!}{%
\begin{tabular}{@{}rrrrr@{}}
\toprule
$d$ & subgroup RLerr @ smallest $n$ & advantage & subgroup RLerr @ largest $n$ & advantage \\
\midrule
8    & $2.265\times10^{-3}$ & $3.2\times$ (Sobol)     & $7.927\times10^{-5}$ & $0.4\times$ (Sobol wins) \\
16   & $2.329\times10^{-3}$ & $4.7\times$ (Gaussian)  & $3.982\times10^{-5}$ & $0.6\times$ (Sobol wins) \\
32   & $6.979\times10^{-4}$ & $8.4\times$ (Halton)    & $6.292\times10^{-6}$ & $3.4\times$ (Halton) \\
64   & $3.972\times10^{-4}$ & $4.2\times$ (Halton)    & $3.610\times10^{-6}$ & $2.6\times$ (Sobol) \\
128  & $5.628\times10^{-4}$ & $1.8\times$ (ORF)       & $2.187\times10^{-6}$ & $2.0\times$ (Halton) \\
256  & $1.969\times10^{-4}$ & $2.0\times$ (Sobol)     & $7.746\times10^{-7}$ & $4.9\times$ (Gaussian) \\
512  & $4.067\times10^{-5}$ & $7.5\times$ (Sobol)     & $4.498\times10^{-7}$ & $3.4\times$ (Sobol) \\
1024 & $2.551\times10^{-5}$ & $6.8\times$ (Halton)    & $1.816\times10^{-7}$ & $4.7\times$ (Sobol) \\
2048 & $9.006\times10^{-6}$ & $13.5\times$ (Gaussian) & $8.951\times10^{-8}$ & $5.5\times$ (Halton) \\
\bottomrule
\end{tabular}
}
\end{table}

\begin{figure}[t]
\centering
\includegraphics[width=0.9\linewidth]{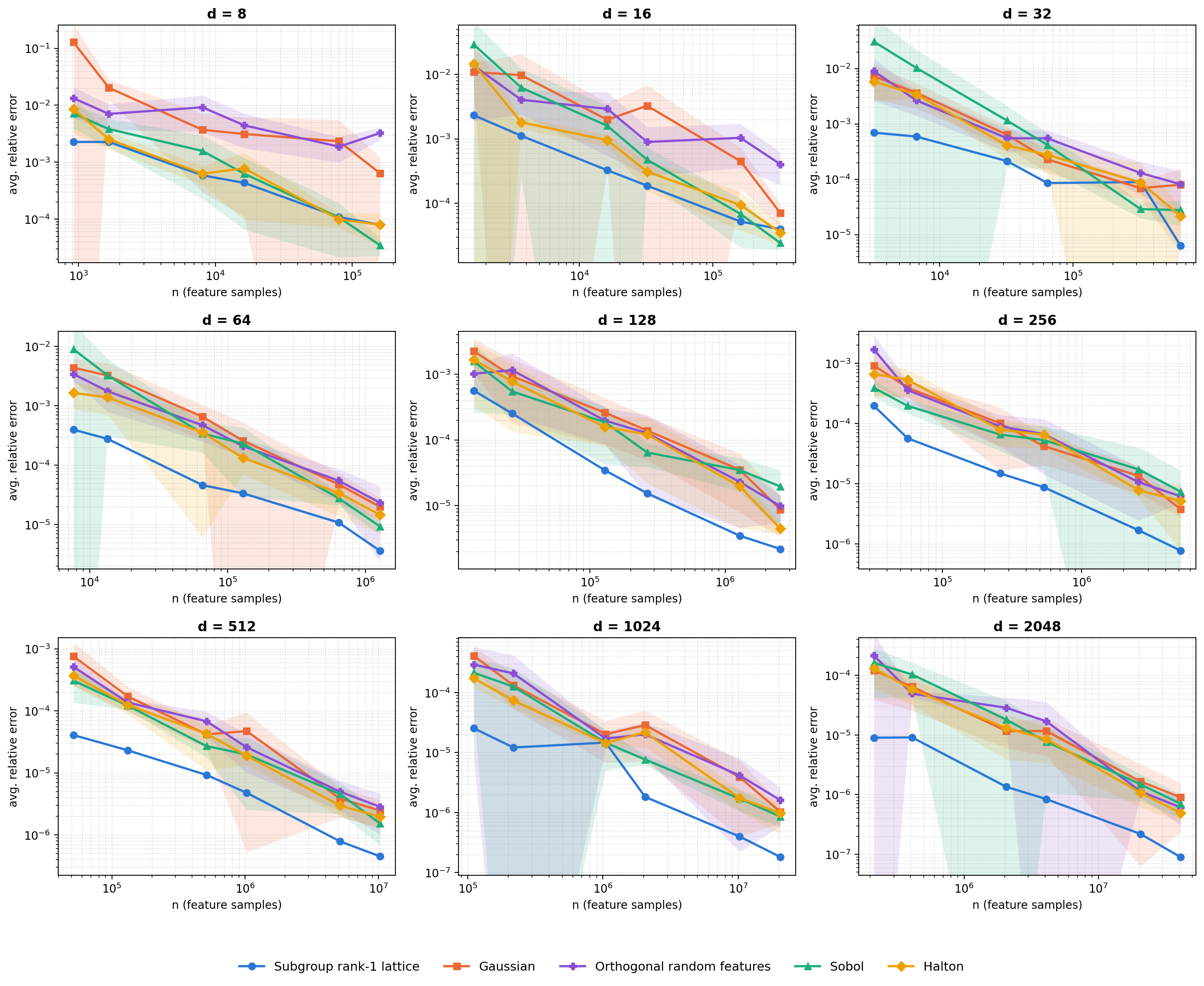}
\caption{Subgroup rank-1 lattice vs.\ Gaussian / ORF / Sobol / Halton:
mean relative error of $\exp(\mathbf x^\top\mathbf y)$ estimation, one
panel per dimension $d\in\{8,\dots,2048\}$, log-log axes, shaded bands
$\pm 1$ standard deviation over repeats.}
\label{fig:synth-sweep}
\end{figure}

\paragraph{Magnitude and scaling of the advantage.} \Cref{tab:synth-advantage}
and \Cref{fig:synth-sweep} show that, at the smallest tested $n$ for a
given $d$, the subgroup lattice's advantage already ranges from roughly
$1.8\times$ to $13.5\times$, the largest being at $d=2048$; at the largest
tested $n$ for $d\ge32$ it remains between $2.0\times$ and $5.5\times$,
again largest at $d=2048,\ n\approx4.1\times10^7$, where the subgroup
lattice's mean relative error of $8.95\times10^{-8}$ is below all four
baselines, including the closest, Halton, at $4.92\times10^{-7}$. The
advantage is not monotone in $d$ at the largest $n$ ($3.4\times$, $2.6\times$,
$2.0\times$, $4.9\times$, $3.4\times$, $4.7\times$, $5.5\times$ for
$d=32,\dots,2048$), but it is at least $2\times$ throughout this range. This is consistent with Theorem~\ref{thm:average}, by the exhaustive search over its small candidate pool ($\varphi(m)=m-1$), the error bound can further reduce at least $O(\frac{1}{m})$. 
Four of the five exceptions
in \Cref{tab:synth-exceptions} sit exactly where this adaptivity has the
least room to work: at $d=8,16,32$ the subgroup order is itself small
($m=17,31,67$ respectively), so the candidate pool $\varphi(m)=m-1$ is
small and the lattice's structural advantage is correspondingly limited,
while the four baselines are already in, or close to, their asymptotic
sampling-density regime at these small dimensions and largest tested $n$.
The fifth exception, at $d=1024$, does not fit this pattern, $m$ is large
there and the lattice wins at every other $n$ tested for that $d$,   and is
a property of this specific point in the sweep rather than of the
subgroup construction generally: its error curve is flat rather than
decreasing between $n_0/d=200$ and $n_0/d=1000$ (mean RLerr
$2.55\times10^{-5}$ then $1.22\times10^{-5}$ then $1.47\times10^{-5}$,
before dropping roughly $8\times$ at the next point), with per-repeat
standard deviations small enough relative to the mean that this is
unlikely to be sampling noise from the $5$ repeats alone. This may due to the deterministic points set construction without randomness as Monte Carlo.  Nevertheless, the subgroup rank-1 lattice is still a near-tie with the best baselines in this case ($1.5\%$ above Halton and $0.5\%$ above Sobol). 

Throughout this paper $m$ is fixed to the \emph{first}
prime such that $m \ge 2d-1$.  We re-ran the single
$(d=1024,\,n_0=1{,}024{,}000)$ configuration with $m$ set instead to the
\emph{second} prime such that $m \ge 2d-1$ ($m=2063$ instead of
$2053$, giving $n=1{,}118{,}147$ instead of $1{,}047{,}031$), together with
a full  exhaustive search over all $m-1=2062$ candidates at
$BN\_search=2000$. The result is a substantially lower subgroup
error, $3.476\times10^{-6}\pm9.1\times10^{-8}$, a $4.2\times$ reduction
from the first-next-prime run's $1.470\times10^{-5}$, which now beats
every baseline at this $(d,n_0)$ by a comfortable margin (Sobol,
the closest baseline, at $1.286\times10^{-5}$: a $3.7\times$ advantage,
versus $8.1\times$ over the worst baseline, Gaussian), turning the earlier
near-loss into a clear win. This single comparison suggests that  a small change of $m$ may mitigate the "not good" lattice  corresponding to the $(d,m,n)$-pair due to the deterministic points set construction.

\subsubsection{Sample set construction time comparison}
\label{sec:exp-construction-time}

The accuracy comparison above holds the number of feature samples $n$
fixed across methods. It does not say what each sample set costs to build.
We therefore time, on the same grid of $54$ configurations
($d\in\{8,\dots,2048\}$, $n_0/d\in\{100,200,1000,2000,10000,20000\}$, the
same $m$, $n$ and cached \texttt{best\_jj} as above), the construction of
each method's complete set of $n$ Gaussianized projection directions in
$\R^d$. No data are involved: nothing is projected, exponentiated or
compared with the kernel, so the timings isolate the cost of the point set
itself.

\paragraph{What each method builds.} For the four baselines the sample set
is a dense $d\times n$ matrix: \texttt{torch.randn}$(d,n)$ for Gaussian;
$n/d$ independent $d\times d$ orthogonal blocks, each obtained from an SVD
and scaled by $\sqrt d$, for ORF; and scrambled Sobol' or Halton points in
$[0,1]^d$ mapped entrywise through $\Phi^{-1}$ for the two QMC baselines.
The subgroup lattice does not need this matrix. Write $n-1=mq$ and let
$c_1,\dots,c_q$ be representatives of the $q$ cosets of $H=\langle
t\rangle$ in $\Fnx$ (\Cref{sec:fast-setting}). Every nonzero lattice point
$l=c_s t^{k}$ has coordinates $X_{l,j}=\{c_s t^{k+j-1}/n\}$,
$j=1,\dots,d$, which is a length-$d$ cyclic window of the single coset
sequence $(\{c_s t^{i}/n\})_{i=0}^{m-1}$; the point $l=0$ maps to
$\Psi(0)=0$. The whole sample set $\Psi(X)$ is therefore determined by the
$m\times q$ array
\begin{equation}
\label{eq:coset-array}
\Delta_{i,s} := \Psi\big((c_s t^{i} \bmod n)/n\big)
= \Phi^{-1}\big(\{c_s t^{i}/n + 0.5\}\big), \qquad
i=0,\dots,m-1,\ \ s=1,\dots,q,
\end{equation}
whose $q$ columns are exactly the coset vectors $\phi^{(1)},\dots,\phi^{(q)}$
consumed by \Cref{alg:fast-Xw} (and, up to reversing the order within each
coset, the vectors $a^{(s)}$ of \Cref{alg:fast-transform}). Building the sample set thus means filling
$mq=n-1$ entries instead of $dn$, with no loss of information: the cyclic
shifts of the columns of $\Delta$ along the $i$ axis recover all $n-1$
nonzero directions. We time exactly this step.

\paragraph{Protocol.} All methods run in \texttt{float64} with PyTorch
2.8.0 on a single NVIDIA H200 NVL GPU (host CPU: AMD EPYC 9655). Each
measurement is one untimed warm-up call followed by $5$ timed repeats,
each bracketed by \texttt{torch.cuda.synchronize()}; we report the mean
and standard deviation over the repeats. Sobol' and Halton points are
generated on the CPU (\texttt{torch.quasirandom.SobolEngine} and
\texttt{scipy.stats.qmc.Halton}, respectively), as in the accuracy
experiments, then copied to the GPU for the $\Phi^{-1}$ map. To bound
memory, every method builds its matrix in chunks of at most $8\times10^8$
entries, discarding each chunk once built. 

\begin{figure}[t]
\centering
\includegraphics[width=0.9\linewidth]{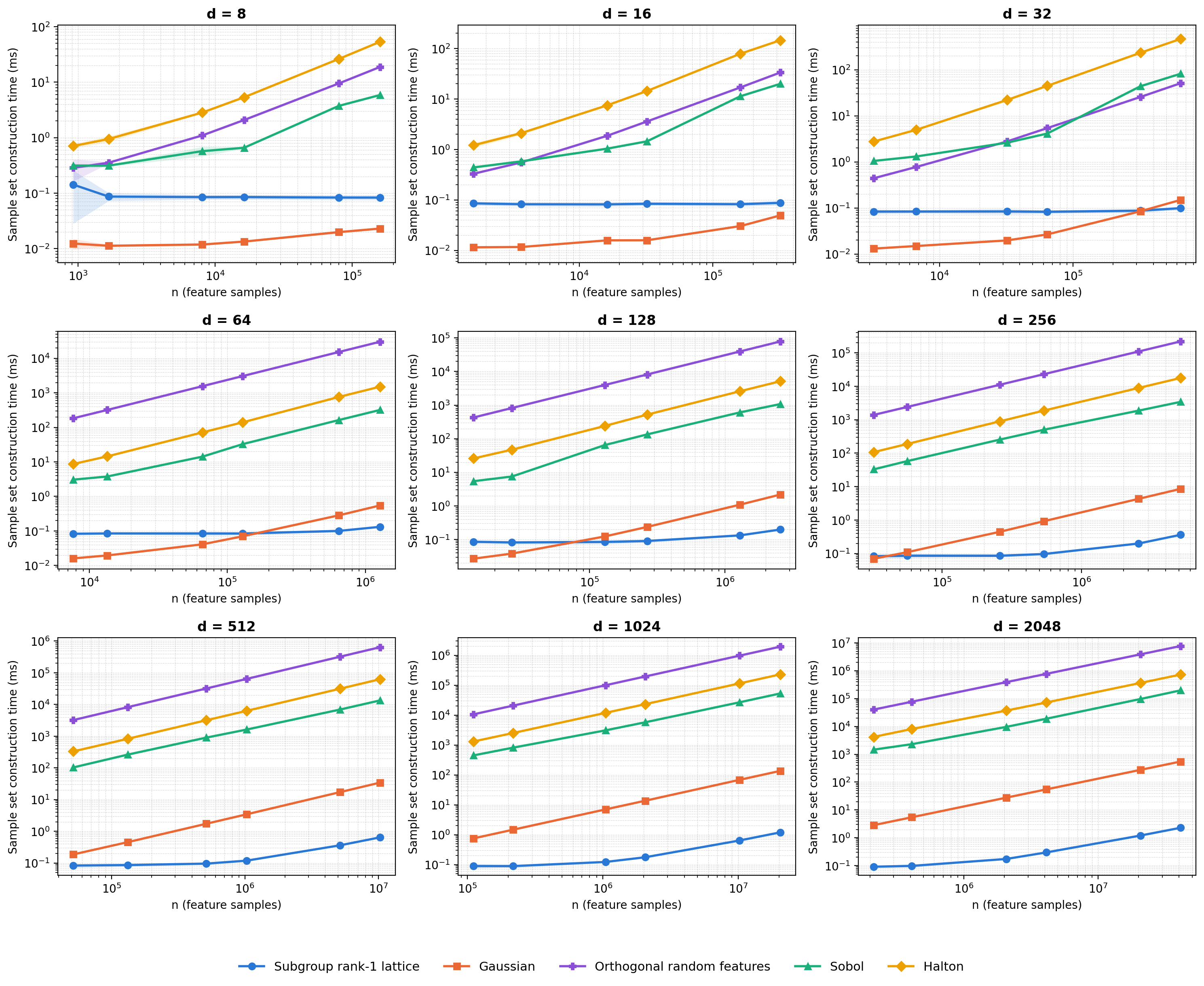}
\caption{Wall-clock time to construct the full sample set of $n$
Gaussianized directions, subgroup rank-1 lattice vs.\ Gaussian / ORF /
Sobol' / Halton, one panel per dimension $d\in\{8,\dots,2048\}$, log-log
axes, shaded bands $\pm1$ standard deviation over $5$ repeats.}
\label{fig:construction-time}
\end{figure}

\begin{table}[t]
\centering
\small
\caption{Sample set construction time at the largest tested $n$ for each
$d$ ($n_0/d=20{,}000$), mean over $5$ repeats. The last column is the
Gaussian time divided by the subgroup time ($>1$ favors the subgroup
lattice).}
\label{tab:construction-time}
\begin{tabular}{@{}rrrrrrrr@{}}
\toprule
$d$ & $n$ & Subgroup & Gaussian & ORF & Sobol' & Halton & Gauss.\,/\,Subgr. \\
\midrule
8    & 160{,}073    & 0.083\,ms & 0.023\,ms & 18.9\,ms & 5.9\,ms  & 53.6\,ms & $0.28\times$ \\
16   & 320{,}107    & 0.088\,ms & 0.049\,ms & 33.5\,ms & 20.0\,ms & 146\,ms  & $0.56\times$ \\
32   & 641{,}057    & 0.099\,ms & 0.150\,ms & 50.9\,ms & 81.8\,ms & 466\,ms  & $1.5\times$ \\
64   & 1{,}280{,}161  & 0.132\,ms & 0.549\,ms & 30.3\,s  & 0.32\,s  & 1.51\,s  & $4.2\times$ \\
128  & 2{,}561{,}263  & 0.198\,ms & 2.18\,ms  & 78.4\,s  & 1.08\,s  & 5.14\,s  & $11\times$ \\
256  & 5{,}130{,}809  & 0.362\,ms & 8.56\,ms  & 220\,s   & 3.43\,s  & 17.7\,s  & $24\times$ \\
512  & 10{,}318{,}249 & 0.646\,ms & 34.4\,ms  & 643\,s   & 13.5\,s  & 63.0\,s  & $53\times$ \\
1024 & 20{,}493{,}047 & 1.20\,ms  & 137\,ms   & 1{,}954\,s & 53.2\,s & 231\,s  & $114\times$ \\
2048 & 41{,}039{,}189 & 2.29\,ms  & 548\,ms   & 7{,}690\,s & 197\,s  & 726\,s  & $239\times$ \\
\bottomrule
\end{tabular}
\end{table}

\paragraph{Headline result.} \Cref{fig:construction-time} and
\Cref{tab:construction-time} show that the subgroup lattice is the fastest
method in every configuration with $d\ge512$, and that its construction
time never exceeds $2.3$\,ms anywhere on the grid. At the largest
configuration ($d=2048$, $n\approx4.1\times10^7$) it builds the sample set
in $2.29$\,ms, against $0.55$\,s for Gaussian ($239\times$ slower),
$197$\,s for Sobol', $726$\,s for Halton, and $7{,}690$\,s ($2.1$ hours)
for ORF. 

\paragraph{Scaling.} The gap follows directly from the entry counts: $n-1$
for the subgroup lattice against $dn$ for every baseline. At fixed $d$,
the log-log slope in $n$ between the two largest tested $n$ is
$0.96$--$1.01$ for Gaussian, ORF and Halton once $d\ge64$, and
$0.83$--$1.04$ for Sobol'. The subgroup lattice instead sits at a floor of
about $0.08$\,ms, which is the fixed cost of launching its few GPU kernels,
up to $n\approx10^5$--$10^6$ (slope $0.0$--$0.2$ for $d\le32$), and
approaches linear growth only at the largest sizes (slope $0.83$--$0.93$
for $d\ge256$). Along the diagonal $n\propto d$ (fitting $n_0/d=1000$ over
$d\in\{256,\dots,2048\}$), Gaussian grows as $d^{2.0}$, matching its
$dn\propto d^2$ entries, Sobol', Halton and ORF as $d^{1.7}$--$d^{1.8}$,
and the subgroup lattice as only $d^{0.34}$, since at these sizes it is
still close to its fixed-cost floor.

\paragraph{Where Gaussian is faster.} Gaussian sampling is a single
\texttt{torch.randn} call and is faster than the subgroup construction
when $dn$ is small: for every $n$ at $d\le16$, for $n\le3.2\times10^5$ at
$d=32$, for $n\le1.3\times10^5$ at $d=64$, for $n\le2.7\times10^4$ at
$d=128$, and only at the smallest $n$ ($3.2\times10^4$) at $d=256$. In all
of these configurations both methods finish in about $0.15$\,ms or less,
so the difference is immaterial in practice. From $d=512$ upward the
subgroup lattice is faster at every $n$, and its lead grows with both $d$
and $n$.

\paragraph{The structured baselines.} ORF is the slowest method for every
$d\ge64$; its cost is dominated by one $d\times d$ SVD per block of $d$
columns. Its time jumps by about $600\times$ between $d=32$ and $d=64$ at
the largest $n$ ($50.9$\,ms to $30.3$\,s), far more than the change in
the number of SVDs. We have not profiled this, but a plausible cause is
that the batched \texttt{float64} SVD in PyTorch uses a fast batched
cuSOLVER path only for matrices up to $32\times32$. Halton and Sobol' are
the next slowest, largely because their points are generated on the CPU
and then copied to the GPU. A GPU-native QMC generator would narrow this
gap, but not close it, since it would still have to produce $dn$ entries.

\paragraph{Noise and scope.} The standard deviation over repeats is below
$7\%$ of the mean for every timing longer than $1$\,s and at most $20\%$
everywhere except the smallest configuration ($d=8$, $n=919$), where the
subgroup ($0.143\pm0.115$\,ms) and ORF ($0.29\pm0.12$\,ms) timings are
dominated by launch overhead at the sub-millisecond floor. These numbers
measure construction only, not the application of the feature map to
data. Construction is a one-time cost that can be reused across data
batches, but for the ORF, Sobol' and Halton baselines at large $d$ and $n$
it runs from minutes to hours per sample set, which is a real cost
whenever $n$ or $d$ changes. All timings come
from a single GPU.

\subsection{Real-data softmax-attention approximation}
\label{sec:exp-realdata}

The second experiment applies the same five feature-map constructions to
\emph{self-normalized softmax attention}
\citep{vaswani2017attention,choromanski2021rethinking} on real,
high-dimensional embedding vectors rather than synthetic uniform data.
For a batch of embedding vectors $\mathbf x_i$ sampled from a dataset and
rescaled by a global norm constant, the ground truth is the exact
row-softmax attention matrix $\mathrm{GT}=\mathrm{softmax}(\boldsymbol{X}\boldsymbol{X}^\top)$; each
method's feature map  is turned into a self-normalized attention-row
estimate by dividing by a feature-based softmax normalizer, matching the
positive-random-feature (FAVOR+-style) construction of
\citet{choromanski2021rethinking}, and the reported error is the mean
squared relative error $\mathrm{mean}\big((|\mathrm{GT}-\widehat{\mathrm{GT}}|
/\min(\mathrm{GT},1-\mathrm{GT}))^2\big)$, normalized to stay well behaved
near the $[0,1]$ boundary of a softmax probability. \Cref{tab:dataset_summary}
lists the nine datasets\footnote{https://github.com/vector-index-bench/vibe} tested, spanning three embedding families (text,
image, and multimodal), dimensions from $128$ to $2048$, and $114$K to
$1.34$M points; for each we test five target ratios $n_0/d\in\{100,200,
300,400,500\}$, for $9\times5=45$ configurations, with the number of
repeats fixed at $n\_seeds=5$ for every configuration.   Evaluation at batch with batchsize 
$BN\_final=5000$  took $\approx1\,\mathrm h\,51\,\mathrm m$ of GPU
time across all nine datasets, dominated by celeba-resnet-2048-cosine
alone ($\approx1\,\mathrm h\,4\,\mathrm m$, its largest $n$ exceeding
$10^6$).  With index \texttt{best\_jj} reused from the  cached
exhaustive search over $m-1$ generator candidate on a held-out batch with batchsize $BN\_search=5000$.

\begin{table}[t]
\centering
\caption{Dataset Summary}
\label{tab:dataset_summary}
\begin{tabular}{lrr}
\toprule
Dataset & Points ($N$) & Dimension ($d$) \\
\midrule
yandex-200-cosine              & 1{,}000{,}000 & 200   \\
laion-clip-512-normalized      & 999{,}448     & 512   \\
arxiv-nomic-768-normalized     & 1{,}344{,}643 & 768   \\
landmark-nomic-768-normalized  & 760{,}757     & 768   \\
imagenet-align-640-normalized  & 315{,}648     & 640   \\
llama-128-ip                   & 255{,}921     & 128   \\
celeba-resnet-2048-cosine      & 201{,}599     & 2{,}048 \\
yi-128-ip                      & 186{,}843     & 128   \\
coco-nomic-768-normalized      & 114{,}054     & 768   \\
\bottomrule
\end{tabular}
\end{table}

\begin{table}[t]
\centering
\small
\caption{Subgroup lattice's advantage over the best of \{Gaussian, ORF,
Sobol, Halton\} on softmax-attention approximation, by dataset, across the
five tested ratios $n_0/d\in\{100,\dots,500\}$.}
\label{tab:realdata-advantage}
\begin{tabular}{@{}lrrrr@{}}
\toprule
Dataset & $d$ & subgroup wins & advantage range & mean advantage \\
\midrule
celeba-resnet-2048-cosine      & 2048 & 5/5 & $1.13$--$1.39\times$ & $1.27\times$ \\
arxiv-nomic-768-normalized     & 768  & 5/5 & $1.57$--$1.84\times$ & $1.70\times$ \\
landmark-nomic-768-normalized  & 768  & 5/5 & $1.95$--$3.37\times$ & $2.83\times$ \\
coco-nomic-768-normalized      & 768  & 5/5 & $1.37$--$1.82\times$ & $1.59\times$ \\
imagenet-align-640-normalized  & 640  & 5/5 & $1.52$--$1.76\times$ & $1.66\times$ \\
laion-clip-512-normalized      & 512  & 5/5 & $1.25$--$1.64\times$ & $1.38\times$ \\
yandex-200-cosine              & 200  & 5/5 & $1.15$--$1.36\times$ & $1.25\times$ \\
llama-128-ip                   & 128  & 5/5 & $1.45$--$2.21\times$ & $1.78\times$ \\
yi-128-ip                      & 128  & 5/5 & $1.43$--$1.83\times$ & $1.60\times$ \\
\bottomrule
\end{tabular}
\end{table}

\begin{figure}[t]
\centering
\includegraphics[width=0.9\linewidth]{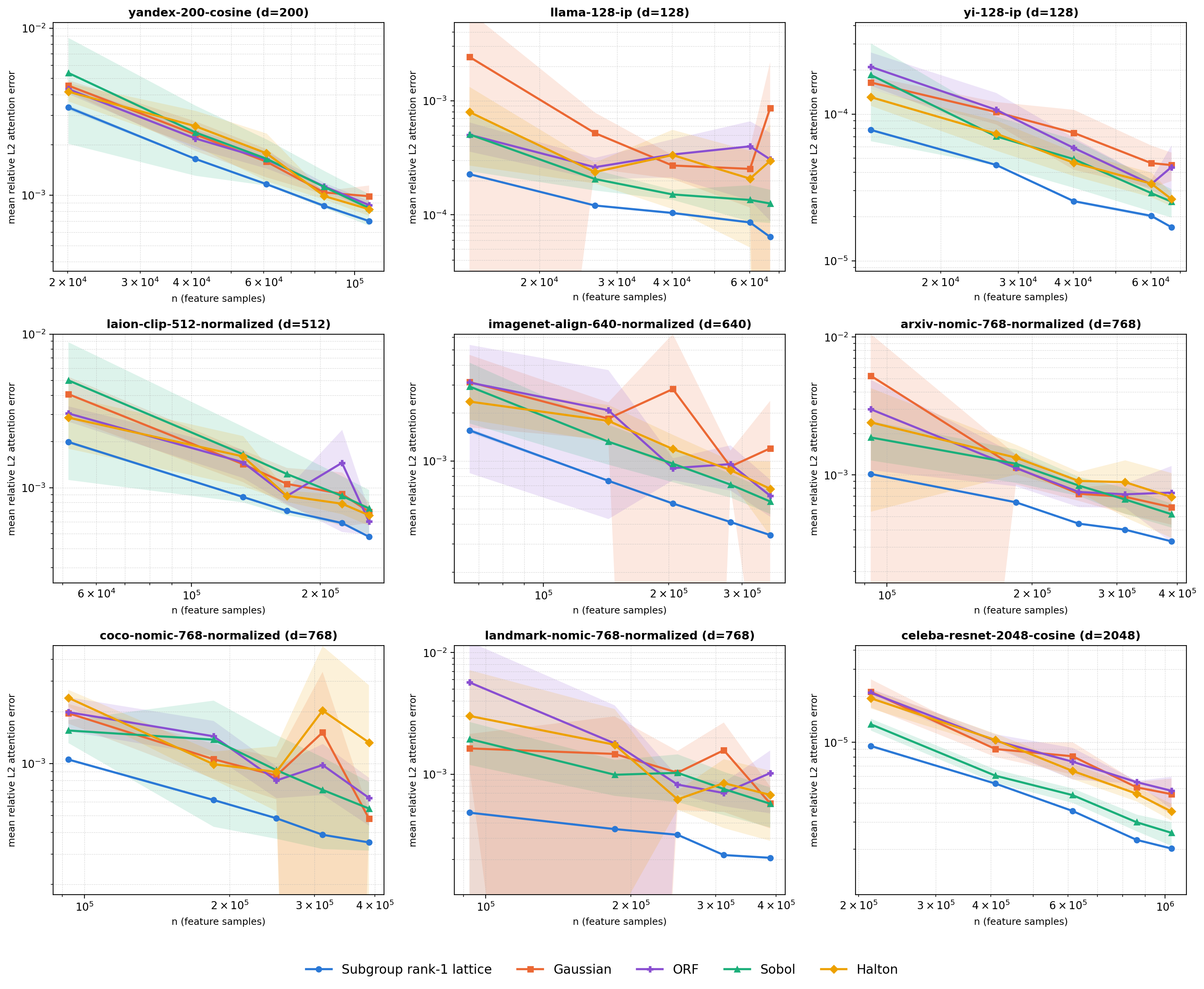}
\caption{Subgroup rank-1 lattice vs.\ Gaussian / ORF / Sobol / Halton:
self-normalized softmax-attention approximation error on all $9$ datasets
of \Cref{tab:dataset_summary}, one panel per dataset, log-log axes,
shaded bands $\pm 1$ standard deviation over repeats.}
\label{fig:dataset-grid}
\end{figure}

\paragraph{Headline result.} The subgroup lattice attains the lowest mean
attention error in all $45$ of $45$ configurations ($100\%$;
\Cref{tab:realdata-advantage}), a stronger result than the $91\%$ win rate
on synthetic data.  Averaged over all
$45$ configurations the advantage ranges from $1.13\times$ to $3.37\times$
(mean $1.67\times$, median $1.57\times$); \Cref{tab:realdata-advantage}
shows no clear monotonic trend of advantage with $d$ over the five ratios range  on different datasets,
landmark-nomic-768-normalized showing the largest mean advantage
($2.83\times$) and yandex-200-cosine the smallest ($1.25\times$, still a
win at every tested $n$).

\paragraph{Runner-up and variance.}  Across all $45$
configurations, which of the four baselines comes closest to the subgroup
lattice: Sobol is closest in $22/45$ ($49\%$), Halton in $10/45$ ($22\%$),
ORF in $8/45$ ($18\%$), and Gaussian in $5/45$ ($11\%$).  Sobol is the
strongest baseline on real data, more consistently so than on the
synthetic benchmark, where Halton was comparably competitive. Gaussian and
Halton show the widest, sometimes non-monotonic, $\pm1$-std bands in
\Cref{fig:dataset-grid}, especially on imagenet-align-640-normalized,
llama-128-ip, and landmark-nomic-768-normalized; the subgroup lattice's
curve is again both the lowest-error and the tightest-banded curve in
every one of the nine panels, echoing \Cref{sec:exp-synthetic}'s finding
that it is more stable, not merely more accurate. ORF's underperformance
is likewise confirmed on this second, independent task: it is the
strongest of the four baselines in only $8/45$ ($18\%$) of configurations,
behind both Sobol and Halton and only modestly ahead of plain Gaussian.

\subsection{Discussion}
\label{sec:exp-discussion}

Taken together, the two experiments show the fast elementwise transform of
\Cref{sec:fast} is not merely an efficient way to compute an existing
feature map, but a route to a \emph{more accurate} one: on both a raw
exponential-kernel estimator and a self-normalized softmax-attention
estimator, on both synthetic and real, high-dimensional embedding data,
selecting a power-form generator from the small, fixed-order candidate
pool of \Cref{def:fixed-order} gives a feature map that is simultaneously
lower-error and lower-variance than Gaussian, ORF, Sobol, or Halton
features at almost every scale tested, with the one systematic exception
(small $d$, very large $n$) plausibly explained by the candidate pool
$\varphi(m)=m-1$ itself being small there. This advantage is not free: it
requires the one-time, cacheable, $O(m)$-candidate exhaustive search
described above, whose cost is dominated by FFT work, a
genuine practical cost, but one paid once per $(d,m,n)$ triple rather than
once per use. 

\section{Conclusion}
\label{sec:conclusion}

We studied subgroup rank-1 lattices: full $n$-point rank-1 lattice rules
whose Korobov power-form generator $\mathbf z(t)=(1,t,\dots,t^{d-1})$ uses
a scalar $t$ of fixed, $n$-independent multiplicative order $m$. The
standard component-by-component construction would never produce such a
generator, because it searches a pool of candidates that grows with $n$.
Our results show that the restriction is nonetheless a useful trade: it
gives up generic freedom in the generating vector and receives algebraic
structure in return, and that structure pays off in three ways.

The first is computational. The elementwise transforms $\Psi(X)^\top v$
and $\Psi(X)w$ that underlie lattice-based feature maps can be computed
exactly in $O(n\log m)$ arithmetic operations, $O(n\log d)$ when
$m=\Theta(d)$, and $O(n)$ memory, for an arbitrary scalar map $\Psi$
(\Cref{prop:fast-fft,prop:fast-Xw}). The baseline for any standard point
set (a Halton or Sobol' sequence, or a rank-1 lattice with a generic
generating vector) is $O(nd)$ time and memory. The saving comes from a
coset decomposition of $\Fnx$ into short cyclic correlations evaluated by
FFT, and it never requires the $n\times d$ matrix $X$ to be stored. In our
measurements the coset-FFT transform matched direct evaluation to
floating-point precision, ran $13\times$ to $35\times$ faster at
$n\approx1.5\times10^6$, and built a complete sample set of
$n\approx4.1\times10^7$ points in $d=2048$ dimensions in $2.3$\,ms, where
the baselines took between $0.55$\,s and $2.1$ hours.

The second is theoretical. Because the candidate pool has only
$\varphi(m)$ members, classical average-case arguments do not apply, so
convergence has to be proved directly. Viewing the aliasing condition as
the vanishing of an integer polynomial at a root of the cyclotomic
polynomial $\Phi_m$ modulo $n$, we showed that for prime $m\ge d+1$ every
admissible generator satisfies
$e^2(Q_n,\mathbf z(t))=O(n^{-(\alpha-1)/(m-1)})$ (\Cref{thm:rate}), and
that the threshold is exact: for $m\le d$ the squared worst-case error is
bounded below by $2$ for every $n$ (\Cref{prop:floor,cor:threshold}).
The complete splitting of $n$ in $\mathbb Q(\zeta_m)$ then shows that a
bad frequency shared by several candidates forces a correspondingly higher
power of $n$ to divide its resultant, which improves the average
worst-case error over the candidate pool by a factor $\Theta(m-1)$
(\Cref{thm:average}).

The third is empirical. Using the best of the $m-1$ admissible generators,
selected once on a held-out batch and cached, the subgroup lattice was the
most accurate of five feature-map constructions in $49$ of $54$ synthetic
$\exp(\mathbf x^\top\mathbf y)$ estimation settings, with an advantage of
up to $13.5\times$ (at $d=2048$), and in all $45$
self-normalized softmax-attention settings on nine real embedding datasets
(mean advantage $1.67\times$), where it also had the narrowest error
bands across repeats on every dataset.

Selecting the generator requires a
one-time search over $m-1$ candidates, which cost up to about five GPU
hours for our largest configuration, although it is paid once per
$(d,m,n)$ and reused afterwards. The construction was least competitive at
small $d$ with very large $n$, where the candidate pool is small, and in
one configuration a poorly suited $(d,m,n)$ triple was fixed simply by
moving to the next admissible prime $m$. Finally, all theory assumes that
both $n$ and $m$ are prime.


\clearpage

\vskip 0.2in
\bibliography{references}

@article{sloan2002component,
  author  = {Sloan, Ian H. and Reztsov, Andrei V.},
  title   = {Component-by-component construction of good lattice rules},
  journal = {Mathematics of Computation},
  volume  = {71},
  number  = {237},
  pages   = {263--274},
  year    = {2002}
}

@article{kuo2003component,
  author  = {Kuo, Frances Y.},
  title   = {Component-by-component constructions achieve the optimal rate of convergence for multivariate integration in weighted {K}orobov and {S}obolev spaces},
  journal = {Journal of Complexity},
  volume  = {19},
  number  = {3},
  pages   = {301--320},
  year    = {2003}
}

@article{nuyens2006fast,
  author  = {Nuyens, Dirk and Cools, Ronald},
  title   = {Fast component-by-component construction of rank-1 lattice rules with a non-prime number of points},
  journal = {Journal of Complexity},
  volume  = {22},
  number  = {1},
  pages   = {4--28},
  year    = {2006}
}

@article{dick2013high,
  author  = {Dick, Josef and Kuo, Frances Y. and Sloan, Ian H.},
  title   = {High-dimensional integration: the quasi-{M}onte {C}arlo way},
  journal = {Acta Numerica},
  volume  = {22},
  pages   = {133--288},
  year    = {2013}
}

@article{korobov1959approximate,
  author  = {Korobov, N. M.},
  title   = {The approximate computation of multiple integrals},
  journal = {Doklady Akademii Nauk SSSR},
  volume  = {124},
  pages   = {1207--1210},
  year    = {1959},
  note    = {(In Russian)}
}

@book{niederreiter1992random,
  author    = {Niederreiter, Harald},
  title     = {Random Number Generation and Quasi-{M}onte {C}arlo Methods},
  series    = {CBMS-NSF Regional Conference Series in Applied Mathematics},
  volume    = {63},
  publisher = {SIAM},
  address   = {Philadelphia, PA},
  year      = {1992}
}

@book{sloanjoe1994lattice,
  author    = {Sloan, Ian H. and Joe, Stephen},
  title     = {Lattice Methods for Multiple Integration},
  publisher = {Oxford University Press},
  address   = {Oxford, UK},
  year      = {1994}
}

@article{halton1960efficiency,
  author  = {Halton, John H.},
  title   = {On the efficiency of certain quasi-random sequences of points in evaluating multi-dimensional integrals},
  journal = {Numerische Mathematik},
  volume  = {2},
  number  = {1},
  pages   = {84--90},
  year    = {1960}
}

@article{sobol1967distribution,
  author  = {Sobol', Ilya M.},
  title   = {On the distribution of points in a cube and the approximate evaluation of integrals},
  journal = {USSR Computational Mathematics and Mathematical Physics},
  volume  = {7},
  number  = {4},
  pages   = {86--112},
  year    = {1967}
}

@book{irelandrosen1990classical,
  author    = {Ireland, Kenneth and Rosen, Michael},
  title     = {A Classical Introduction to Modern Number Theory},
  edition   = {2nd},
  series    = {Graduate Texts in Mathematics},
  volume    = {84},
  publisher = {Springer-Verlag},
  address   = {New York},
  year      = {1990}
}

@article{avron2016quasi,
  author  = {Avron, Haim and Sindhwani, Vikas and Yang, Jiyan and Mahoney, Michael W.},
  title   = {Quasi-{M}onte {C}arlo Feature Maps for Shift-Invariant Kernels},
  journal = {Journal of Machine Learning Research},
  volume  = {17},
  number  = {120},
  pages   = {1--38},
  year    = {2016}
}

@book{washington1997introduction,
  author    = {Washington, Lawrence C.},
  title     = {Introduction to Cyclotomic Fields},
  edition   = {2nd},
  series    = {Graduate Texts in Mathematics},
  volume    = {83},
  publisher = {Springer-Verlag},
  address   = {New York},
  year      = {1997}
}

@article{cooley1965algorithm,
  author  = {Cooley, James W. and Tukey, John W.},
  title   = {An algorithm for the machine calculation of complex {F}ourier series},
  journal = {Mathematics of Computation},
  volume  = {19},
  number  = {90},
  pages   = {297--301},
  year    = {1965}
}

@article{roberts1996exponential,
  author  = {Roberts, Gareth O. and Tweedie, Richard L.},
  title   = {Exponential convergence of {L}angevin distributions and their discrete approximations},
  journal = {Bernoulli},
  volume  = {2},
  number  = {4},
  pages   = {341--363},
  year    = {1996}
}

@inproceedings{welling2011bayesian,
  author    = {Welling, Max and Teh, Yee Whye},
  title     = {Bayesian Learning via Stochastic Gradient {L}angevin Dynamics},
  booktitle = {Proceedings of the 28th International Conference on Machine Learning (ICML-11)},
  pages     = {681--688},
  year      = {2011}
}

@book{robert2004monte,
  author    = {Robert, Christian P. and Casella, George},
  title     = {Monte Carlo Statistical Methods},
  edition   = {2nd},
  series    = {Springer Texts in Statistics},
  publisher = {Springer-Verlag},
  address   = {New York},
  year      = {2004}
}

@article{hickernell1998generalized,
  author  = {Hickernell, Fred J.},
  title   = {A Generalized Discrepancy and Quadrature Error Bound},
  journal = {Mathematics of Computation},
  volume  = {67},
  number  = {221},
  pages   = {299--322},
  year    = {1998}
}

@book{niederreiterxing2001rational,
  author    = {Niederreiter, Harald and Xing, Chaoping},
  title     = {Rational Points on Curves over Finite Fields: Theory and Applications},
  series    = {London Mathematical Society Lecture Note Series},
  number    = {285},
  publisher = {Cambridge University Press},
  address   = {Cambridge, UK},
  year      = {2001}
}

@book{dickpillichshammer2010digital,
  author    = {Dick, Josef and Pillichshammer, Friedrich},
  title     = {Digital Nets and Sequences: Discrepancy Theory and Quasi-{M}onte {C}arlo Integration},
  publisher = {Cambridge University Press},
  address   = {Cambridge, UK},
  year      = {2010}
}

@incollection{lecuyerlemieux2002recent,
  author    = {L'Ecuyer, Pierre and Lemieux, Christiane},
  title     = {Recent Advances in Randomized Quasi-{M}onte {C}arlo Methods},
  booktitle = {Modeling Uncertainty: An Examination of Stochastic Theory, Methods, and Applications},
  editor    = {Dror, Moshe and L'Ecuyer, Pierre and Szidarovszky, Ferenc},
  series    = {International Series in Operations Research \& Management Science},
  volume    = {46},
  publisher = {Kluwer Academic Publishers},
  pages     = {419--474},
  year      = {2002}
}

@article{joekuo2008constructing,
  author  = {Joe, Stephen and Kuo, Frances Y.},
  title   = {Constructing {S}obol Sequences with Better Two-Dimensional Projections},
  journal = {SIAM Journal on Scientific Computing},
  volume  = {30},
  number  = {5},
  pages   = {2635--2654},
  year    = {2008}
}

@article{joekuo2003remark,
  author  = {Joe, Stephen and Kuo, Frances Y.},
  title   = {Remark on {A}lgorithm 659: Implementing {S}obol's Quasirandom Sequence Generator},
  journal = {ACM Transactions on Mathematical Software},
  volume  = {29},
  number  = {1},
  pages   = {49--57},
  year    = {2003}
}

@article{coolskuonuyens2006constructing,
  author  = {Cools, Ronald and Kuo, Frances Y. and Nuyens, Dirk},
  title   = {Constructing Embedded Lattice Rules for Multivariate Integration},
  journal = {SIAM Journal on Scientific Computing},
  volume  = {28},
  number  = {6},
  pages   = {2162--2188},
  year    = {2006}
}

@article{dickkuopillichshammersloan2005construction,
  author  = {Dick, Josef and Kuo, Frances Y. and Pillichshammer, Friedrich and Sloan, Ian H.},
  title   = {Construction Algorithms for Polynomial Lattice Rules for Multivariate Integration},
  journal = {Mathematics of Computation},
  volume  = {74},
  number  = {252},
  pages   = {1895--1921},
  year    = {2005}
}

@article{nuyenscools2006fastrkhs,
  author  = {Nuyens, Dirk and Cools, Ronald},
  title   = {Fast Algorithms for Component-by-Component Construction of Rank-1 Lattice Rules in Shift-Invariant Reproducing Kernel Hilbert Spaces},
  journal = {Mathematics of Computation},
  volume  = {75},
  number  = {254},
  pages   = {903--920},
  year    = {2006}
}

@article{sloanwozniakowski1998when,
  author  = {Sloan, Ian H. and Wo{\'z}niakowski, Henryk},
  title   = {When Are Quasi-{M}onte {C}arlo Algorithms Efficient for High Dimensional Integrals?},
  journal = {Journal of Complexity},
  volume  = {14},
  number  = {1},
  pages   = {1--33},
  year    = {1998}
}

@article{niederreiter1987point,
  author  = {Niederreiter, Harald},
  title   = {Point Sets and Sequences with Small Discrepancy},
  journal = {Monatshefte f{\"u}r Mathematik},
  volume  = {104},
  number  = {4},
  pages   = {273--337},
  year    = {1987}
}

@article{faure1982discrepance,
  author  = {Faure, Henri},
  title   = {Discr{\'e}pance de suites associ{\'e}es {\`a} un syst{\`e}me de num{\'e}ration (en dimension s)},
  journal = {Acta Arithmetica},
  volume  = {41},
  number  = {4},
  pages   = {337--351},
  year    = {1982}
}

@incollection{owen1995randomly,
  author    = {Owen, Art B.},
  title     = {Randomly Permuted (t,m,s)-Nets and (t,s)-Sequences},
  booktitle = {Monte Carlo and Quasi-Monte Carlo Methods in Scientific Computing},
  editor    = {Niederreiter, Harald and Shiue, Peter Jau-Shyong},
  series    = {Lecture Notes in Statistics},
  volume    = {106},
  publisher = {Springer},
  address   = {New York, NY},
  pages     = {299--317},
  year      = {1995}
}

@article{owen1997scrambled,
  author  = {Owen, Art B.},
  title   = {Scrambled Net Variance for Integrals of Smooth Functions},
  journal = {The Annals of Statistics},
  volume  = {25},
  number  = {4},
  pages   = {1541--1562},
  year    = {1997}
}

@article{matousek1998l2discrepancy,
  author  = {Matou{\v{s}}ek, Ji{\v{r}}{\'\i}},
  title   = {On the {$L_2$}-Discrepancy for Anchored Boxes},
  journal = {Journal of Complexity},
  volume  = {14},
  number  = {4},
  pages   = {527--556},
  year    = {1998}
}

@article{dick2008walsh,
  author  = {Dick, Josef},
  title   = {Walsh Spaces Containing Smooth Functions and Quasi-{M}onte {C}arlo Rules of Arbitrary High Order},
  journal = {SIAM Journal on Numerical Analysis},
  volume  = {46},
  number  = {3},
  pages   = {1519--1553},
  year    = {2008}
}

@article{weyl1916gleichverteilung,
  author  = {Weyl, Hermann},
  title   = {{\"U}ber die Gleichverteilung von Zahlen mod. Eins},
  journal = {Mathematische Annalen},
  volume  = {77},
  number  = {3},
  pages   = {313--352},
  year    = {1916}
}

@article{roth1954irregularities,
  author  = {Roth, Klaus F.},
  title   = {On Irregularities of Distribution},
  journal = {Mathematika},
  volume  = {1},
  number  = {2},
  pages   = {73--79},
  year    = {1954}
}

@book{matousek1999geometric,
  author    = {Matou{\v{s}}ek, Ji{\v{r}}{\'\i}},
  title     = {Geometric Discrepancy: An Illustrated Guide},
  series    = {Algorithms and Combinatorics},
  volume    = {18},
  publisher = {Springer},
  address   = {Berlin, Heidelberg},
  year      = {1999}
}

@article{dick2011higherorder,
  author  = {Dick, Josef},
  title   = {Higher Order Scrambled Digital Nets Achieve the Optimal Rate of the Root Mean Square Error for Smooth Integrands},
  journal = {The Annals of Statistics},
  volume  = {39},
  number  = {3},
  pages   = {1372--1398},
  year    = {2011}
}

@article{dick2016discrepancy,
  author  = {Dick, Josef and Rudolf, Daniel and Zhu, Houying},
  title   = {Discrepancy Bounds for Uniformly Ergodic {M}arkov Chain Quasi-{M}onte {C}arlo},
  journal = {Annals of Applied Probability},
  volume  = {26},
  number  = {5},
  pages   = {3178--3205},
  year    = {2016}
}

@article{buchholz2019improving,
  author  = {Buchholz, Alexander and Chopin, Nicolas},
  title   = {Improving Approximate {B}ayesian Computation via Quasi-{M}onte {C}arlo},
  journal = {Journal of Computational and Graphical Statistics},
  volume  = {28},
  number  = {1},
  pages   = {205--219},
  year    = {2019}
}

@article{gerber2015sequential,
  author  = {Gerber, Mathieu and Chopin, Nicolas},
  title   = {Sequential Quasi {M}onte {C}arlo},
  journal = {Journal of the Royal Statistical Society: Series B (Statistical Methodology)},
  volume  = {77},
  number  = {3},
  pages   = {509--579},
  year    = {2015}
}

@book{leobacher2014introduction,
  author    = {Leobacher, Gunther and Pillichshammer, Friedrich},
  title     = {Introduction to Quasi-{M}onte {C}arlo Integration and Applications},
  series    = {Compact Textbooks in Mathematics},
  publisher = {Birkh{\"a}user},
  address   = {Cham, Switzerland},
  year      = {2014}
}

@article{paskov1995faster,
  author  = {Paskov, Spassimir H. and Traub, Joseph F.},
  title   = {Faster Valuation of Financial Derivatives},
  journal = {The Journal of Portfolio Management},
  volume  = {22},
  number  = {1},
  pages   = {113--120},
  year    = {1995}
}

@article{caflisch1997valuation,
  author  = {Caflisch, Russel E. and Morokoff, William and Owen, Art B.},
  title   = {Valuation of Mortgage-Backed Securities Using {B}rownian Bridges to Reduce Effective Dimension},
  journal = {Journal of Computational Finance},
  volume  = {1},
  number  = {1},
  pages   = {27--46},
  year    = {1997}
}

@article{kuo2017multilevel,
  author  = {Kuo, Frances Y. and Scheichl, Robert and Schwab, Christoph and Sloan, Ian H. and Ullmann, Elisabeth},
  title   = {Multilevel Quasi-{M}onte {C}arlo Methods for Lognormal Diffusion Problems},
  journal = {Mathematics of Computation},
  volume  = {86},
  number  = {308},
  pages   = {2827--2860},
  year    = {2017}
}

@inproceedings{lecuyer2004quasi,
  author    = {L'Ecuyer, Pierre},
  title     = {Quasi-{M}onte {C}arlo Methods in Finance},
  booktitle = {Proceedings of the 2004 Winter Simulation Conference},
  pages     = {1645--1655},
  publisher = {IEEE},
  year      = {2004}
}

@inproceedings{yang2014random,
  author    = {Yang, Jiyan and Sindhwani, Vikas and Fan, Quanfu and Avron, Haim and Mahoney, Michael W.},
  title     = {Random {L}aplace Feature Maps for Semigroup Kernels on Histograms},
  booktitle = {Proceedings of the IEEE Conference on Computer Vision and Pattern Recognition (CVPR)},
  pages     = {971--978},
  year      = {2014}
}

@article{owen2005quasi,
  author  = {Owen, Art B. and Tribble, Seth D.},
  title   = {A Quasi-{M}onte {C}arlo {M}etropolis Algorithm},
  journal = {Proceedings of the National Academy of Sciences},
  volume  = {102},
  number  = {25},
  pages   = {8844--8849},
  year    = {2005}
}

@article{kuo2012qmcfem,
  author  = {Kuo, Frances Y. and Schwab, Christoph and Sloan, Ian H.},
  title   = {Quasi-{M}onte {C}arlo Finite Element Methods for a Class of Elliptic Partial Differential Equations with Random Coefficients},
  journal = {SIAM Journal on Numerical Analysis},
  volume  = {50},
  number  = {6},
  pages   = {3351--3374},
  year    = {2012}
}

@article{graham2011qmcpde,
  author  = {Graham, Ivan G. and Kuo, Frances Y. and Nuyens, Dirk and Scheichl, Robert and Sloan, Ian H.},
  title   = {Quasi-{M}onte {C}arlo Methods for Elliptic {PDE}s with Random Coefficients and Applications},
  journal = {Journal of Computational Physics},
  volume  = {230},
  number  = {10},
  pages   = {3668--3694},
  year    = {2011}
}

@article{gilbert2024eigenvalue,
  author  = {Gilbert, Alexander D. and Scheichl, Robert},
  title   = {Multilevel Quasi-{M}onte {C}arlo for Random Elliptic Eigenvalue Problems {I}: Regularity and Error Analysis},
  journal = {IMA Journal of Numerical Analysis},
  volume  = {44},
  number  = {1},
  pages   = {466--503},
  year    = {2024}
}

@article{kuo2016survey,
  author  = {Kuo, Frances Y. and Nuyens, Dirk},
  title   = {Application of Quasi-{M}onte {C}arlo Methods to Elliptic {PDE}s with Random Diffusion Coefficients: A Survey of Analysis and Implementation},
  journal = {Foundations of Computational Mathematics},
  volume  = {16},
  number  = {6},
  pages   = {1631--1696},
  year    = {2016}
}

@article{hickernell2000integration,
  author  = {Hickernell, Fred J. and Wo{\'z}niakowski, Henryk},
  title   = {Integration and Approximation in Arbitrary Dimensions},
  journal = {Advances in Computational Mathematics},
  volume  = {12},
  number  = {1},
  pages   = {25--58},
  year    = {2000}
}

@book{novak2008tractability,
  author    = {Novak, Erich and Wo{\'z}niakowski, Henryk},
  title     = {Tractability of Multivariate Problems, Volume {I}: Linear Information},
  series    = {EMS Tracts in Mathematics},
  volume    = {6},
  publisher = {European Mathematical Society},
  address   = {Z{\"u}rich, Switzerland},
  year      = {2008}
}

@article{sloan2001tractability,
  author  = {Sloan, Ian H. and Wo{\'z}niakowski, Henryk},
  title   = {Tractability of Multivariate Integration for Weighted {K}orobov Classes},
  journal = {Journal of Complexity},
  volume  = {17},
  number  = {4},
  pages   = {697--721},
  year    = {2001}
}

@article{kuo2010decompositions,
  author  = {Kuo, Frances Y. and Sloan, Ian H. and Wasilkowski, Grzegorz W. and Wo{\'z}niakowski, Henryk},
  title   = {On Decompositions of Multivariate Functions},
  journal = {Mathematics of Computation},
  volume  = {79},
  number  = {270},
  pages   = {953--966},
  year    = {2010}
}

@article{kuo2011anziam,
  author  = {Kuo, Frances Y. and Schwab, Christoph and Sloan, Ian H.},
  title   = {Quasi-{M}onte {C}arlo Methods for High-Dimensional Integration: The Standard (Weighted {H}ilbert Space) Setting and Beyond},
  journal = {The ANZIAM Journal},
  volume  = {53},
  number  = {1},
  pages   = {1--37},
  year    = {2011}
}

@article{dick2014nonperiodic,
  author  = {Dick, Josef and Nuyens, Dirk and Pillichshammer, Friedrich},
  title   = {Lattice Rules for Nonperiodic Smooth Integrands},
  journal = {Numerische Mathematik},
  volume  = {126},
  number  = {2},
  pages   = {259--291},
  year    = {2014}
}

@article{kammerer2015approximation,
  author  = {K{\"a}mmerer, Lutz and Potts, Daniel and Volkmer, Toni},
  title   = {Approximation of Multivariate Periodic Functions by Trigonometric Polynomials Based on Rank-1 Lattice Sampling},
  journal = {Journal of Complexity},
  volume  = {31},
  number  = {4},
  pages   = {543--576},
  year    = {2015}
}

@article{kammerer2015korobov,
  author  = {K{\"a}mmerer, Lutz and Potts, Daniel and Volkmer, Toni},
  title   = {Approximation of Multivariate Periodic Functions by Trigonometric Polynomials Based on Sampling Along Rank-1 Lattice with Generating Vector of {K}orobov Form},
  journal = {Journal of Complexity},
  volume  = {31},
  number  = {3},
  pages   = {424--456},
  year    = {2015}
}

@phdthesis{kammerer2014high,
  author = {K{\"a}mmerer, Lutz},
  title  = {High Dimensional Fast Fourier Transform Based on Rank-1 Lattice Sampling},
  school = {Technische Universit{\"a}t Chemnitz},
  year   = {2014}
}

@article{kammerer2021highdim,
  author  = {K{\"a}mmerer, Lutz and Potts, Daniel and Volkmer, Toni},
  title   = {High-Dimensional Sparse {FFT} Based on Sampling Along Multiple Rank-1 Lattices},
  journal = {Applied and Computational Harmonic Analysis},
  volume  = {51},
  pages   = {225--257},
  year    = {2021}
}

@article{cools2021fast,
  author  = {Cools, Ronald and Kuo, Frances Y. and Nuyens, Dirk and Sloan, Ian H.},
  title   = {Fast Component-by-Component Construction of Lattice Algorithms for Multivariate Approximation with {POD} and {SPOD} Weights},
  journal = {Mathematics of Computation},
  volume  = {90},
  pages   = {787--812},
  year    = {2021}
}

@article{kuo2009lattice,
  author  = {Kuo, Frances Y. and Wasilkowski, Grzegorz W. and Wo{\'z}niakowski, Henryk},
  title   = {Lattice Algorithms for Multivariate {$L_\infty$} Approximation in the Worst Case Setting},
  journal = {Constructive Approximation},
  volume  = {30},
  pages   = {475--493},
  year    = {2009}
}

@book{cools2016monte,
  editor    = {Cools, Ronald and Nuyens, Dirk},
  title     = {Monte Carlo and Quasi-Monte Carlo Methods: {MCQMC}, Leuven, Belgium, April 2014},
  series    = {Springer Proceedings in Mathematics \& Statistics},
  volume    = {163},
  publisher = {Springer},
  year      = {2016}
}

@article{lecuyer2016latticebuilder,
  author    = {L'Ecuyer, Pierre and Munger, David},
  title     = {Algorithm 958: {L}attice{B}uilder: A General Software Tool for Constructing Rank-1 Lattice Rules},
  journal   = {ACM Transactions on Mathematical Software},
  volume    = {42},
  number    = {2},
  articleno = {15},
  year      = {2016}
}

@inproceedings{lyu2020subgroup,
  author    = {Lyu, Yueming and Yuan, Yuan and Tsang, Ivor W.},
  title     = {Subgroup-based Rank-1 Lattice Quasi-{M}onte {C}arlo},
  booktitle = {Advances in Neural Information Processing Systems 33 (NeurIPS 2020)},
  year      = {2020}
}

@article{apostol1970resultants,
  author  = {Apostol, Tom M.},
  title   = {Resultants of Cyclotomic Polynomials},
  journal = {Proceedings of the American Mathematical Society},
  volume  = {24},
  number  = {3},
  pages   = {457--462},
  year    = {1970}
}

@inproceedings{yu2016orthogonal,
  author    = {Yu, Felix X. and Suresh, Ananda Theertha and Choromanski, Krzysztof M. and Holtmann-Rice, Daniel N. and Kumar, Sanjiv},
  title     = {Orthogonal Random Features},
  booktitle = {Advances in Neural Information Processing Systems 29 (NeurIPS 2016)},
  pages     = {1975--1983},
  year      = {2016}
}

@inproceedings{choromanski2021rethinking,
  author    = {Choromanski, Krzysztof and Likhosherstov, Valerii and Dohan, David and Song, Xingyou and Gane, Andreea and Sarlos, Tamas and Hawkins, Peter and Davis, Jared and Mohiuddin, Afroz and Kaiser, Lukasz and Belanger, David and Colwell, Lucy and Weller, Adrian},
  title     = {Rethinking Attention with {P}erformers},
  booktitle = {International Conference on Learning Representations (ICLR)},
  year      = {2021}
}

@inproceedings{vaswani2017attention,
  author    = {Vaswani, Ashish and Shazeer, Noam and Parmar, Niki and Uszkoreit, Jakob and Jones, Llion and Gomez, Aidan N. and Kaiser, {\L}ukasz and Polosukhin, Illia},
  title     = {Attention Is All You Need},
  booktitle = {Advances in Neural Information Processing Systems 30 (NeurIPS 2017)},
  year      = {2017}
}

@inproceedings{rahimi2007random,
  author    = {Rahimi, Ali and Recht, Benjamin},
  title     = {Random Features for Large-Scale Kernel Machines},
  booktitle = {Advances in Neural Information Processing Systems 20 (NIPS 2007)},
  pages     = {1177--1184},
  year      = {2007}
}

@inproceedings{le2013fastfood,
  author    = {Le, Quoc and Sarl{\'o}s, Tam{\'a}s and Smola, Alexander},
  title     = {Fastfood -- Approximating Kernel Expansions in Loglinear Time},
  booktitle = {Proceedings of the 30th International Conference on Machine Learning (ICML)},
  year      = {2013}
}

\end{document}